\documentclass[letterpaper]{article} 
\usepackage{aaai2026}  
\usepackage{times}  
\usepackage{helvet}  
\usepackage{courier}  
\usepackage[hyphens]{url}  
\usepackage{graphicx} 
\usepackage{natbib}  
\usepackage{caption} 
\usepackage{algorithm}
\usepackage{algorithmic}

\usepackage{newfloat}
\usepackage{listings}
\DeclareCaptionStyle{ruled}{labelfont=normalfont,labelsep=colon,strut=off} 
\floatstyle{ruled}
\newfloat{listing}{tb}{lst}{}
\floatname{listing}{Listing}
\usepackage{booktabs}
\usepackage{comment}
\usepackage{subcaption}
\usepackage{amsmath}
\usepackage{amssymb}
\usepackage{mathtools}
\usepackage{amsthm}
\usepackage{url}
\usepackage{tikz}
\usepackage{ifthen,amssymb}
\usepackage{mathrsfs}
\usepackage{cite}
\usepackage{amsfonts,mathtools}
\usepackage{enumerate}
\usepackage{amsthm}
\usepackage{multirow}
\usepackage{mathtools}
\usepackage{amsmath}
\usepackage{cases}
\usepackage{appendix}
\usepackage{color}
\usepackage{bm}
\usepackage{algorithm, algorithmic}
\usepackage{xcolor}
\usepackage{cancel}
\usepackage{microtype}
\usepackage{enumitem}
\usepackage{verbatim}
\setlist{noitemsep}
\usepackage{caption}
\usepackage{graphicx}
\usepackage{makecell}
\usepackage{array}
\usepackage{floatflt}
\usetikzlibrary{shapes, arrows.meta, positioning, decorations.pathreplacing,calc}

\newcommand{\E}[0]{\mathbb{E}}
\newcommand{\R}[0]{\mathbb{R}}

\newcommand{\lossp}[0]{\mathcal{L}}
\newcommand{\loss}[0]{\mathcal{L}_{\text{pref}}}

\newcommand{\D}[0]{\mathcal{D}_{\text{pref}}}
\newcommand{\setA}[0]{\mathcal{A}}

\newcommand{\xsel}[1][t]{g_{\text{sel}}}

\newcommand{\Hess}[0]{\mathbf {H}}
\newcommand{\Dess}[0]{\mathbf {D}}
\newcommand{\Wess}[0]{\mathbf {W}}
\newcommand{\Less}[0]{\mathbf {L}}

\setlist[itemize]{itemsep=1pt, topsep=1pt}
\usetikzlibrary{decorations.pathreplacing,calc}

\usepackage[capitalize,noabbrev]{cleveref}

\theoremstyle{plain}
\newtheorem{theorem}{Theorem}[section]

\newtheorem{lemma}[theorem]{Lemma}

\theoremstyle{definition}

\newtheorem{assumption}[theorem]{Assumption}
\theoremstyle{remark}
\newtheorem{remark}[theorem]{Remark}

\title{Fusing Rewards and Preferences \\in Reinforcement Learning}

\author {
    Sadegh Khorasani\textsuperscript{\rm 1},
    Saber Salehkaleybar\textsuperscript{\rm 2},
    Negar Kiyavash\textsuperscript{\rm 3},
    Matthias Grossglauser\textsuperscript{\rm 1}.
}
\affiliations {
    \textsuperscript{\rm 1}School of Computer and Communication Sciences, EPFL, Lausanne, Switzerland\\
    \textsuperscript{\rm 2}Leiden Institute of Advanced Computer Science (LIACS), Leiden University, Leiden, The Netherlands\\
    \textsuperscript{\rm 3}College of Management of Technology, EPFL, Lausanne, Switzerland\\
    
    {sadegh.khorasani@epfl.ch}, {s.salehkaleybar@liacs.leidenuniv.nl}, {negar.kiyavash@epfl.ch}, {matthias.grossglauser@epfl.ch}
    
}

\usepackage{bibentry}

\begin{document}

\maketitle

\begin{abstract}
We present Dual-Feedback Actor (DFA), 
a reinforcement learning algorithm that fuses both individual rewards and pairwise preferences (if available) into a single update rule.
DFA uses the policy’s log-probabilities directly to model the preference probability, avoiding a separate reward-modeling step. 
Preferences can be provided by human-annotators (at state-level or trajectory-level) or be synthesized online from Q-values stored in an off-policy replay buffer.
Under a Bradley–Terry model, we prove that minimizing DFA’s preference loss recovers the entropy-regularized Soft Actor-Critic (SAC) policy. 
Our simulation results show that DFA trained on generated preferences matches or exceeds SAC on six control environments and demonstrates a more stable training process.  
With only a semi-synthetic preference dataset under Bradley-Terry model, our algorithm outperforms reward-modeling reinforcement learning from human feedback (RLHF)  baselines in a stochastic GridWorld and approaches the performance of an oracle with true rewards. 
\end{abstract}

\section{Introduction}
Over the past decade, Reinforcement Learning (RL) has achieved remarkable success across a wide range of applications, including video games \citep{knox2008tamer,warnell2018deep}, recommendation systems \citep{kohli2013fast,zeng2016online}, and autonomous driving \citep{kiran2021deep}. RL focuses on how agents make decisions while interacting with dynamic, changing environments. 
At each time step, an agent chooses an action based on its current state and receives a reward that indicates how good that action was. The goal is to learn a policy that maximizes the total reward accumulated over time. In traditional RL, the reward function is usually manually designed by experts to guide the agent's behavior toward desired outcomes. However, crafting such a function is a challenging and often ambiguous task \citep{ng2000algorithms}.

To overcome the limitations of hand-engineered rewards, Reinforcement Learning from Human Feedback (RLHF) has emerged as a compelling alternative, particularly in the fine-tuning of large language models (LLMs) \citep{christiano2017deep, stiennon2020learning, ouyang2022training}. RLHF bypasses manual reward specification by inferring a reward model from human preferences over trajectory pairs. This reward model then guides policy optimization using standard RL algorithms. Despite its empirical successes, RLHF methods relying on reward inference, face significant practical and theoretical challenges, including reward model misspecification, overfitting, distribution shift, and non-identifiability of reward functions \citep{zhu2024iterative,casper2023open}. Moreover, the reward inference step introduces additional complexity and often requires large volumes of annotated data. 

To simplify the pipeline and avoid reward inference, in the context of language modeling, Direct Preference Optimization (DPO) has recently been proposed as a direct approach to exploit human preferences \citep{rafailov2023direct}. Thanks to a closed-form expression of the optimal policy under a Bradley–Terry preference model, DPO avoids estimating the reward function. Although DPO has shown promising results in fine-tuning large language models, its loss formulation tends to induce deterministic policies and is susceptible to mode collapse \citep{azar2024general,sharifnassab2024soft}. 
Moreover, the existing theory for DPO only covers contextual bandits or MDPs with deterministic transitions \citep{rafailov2023direct,rafailov2024r}.  
As a result, directly applying DPO (or methods suggested in \citet{guo2024direct, xie2024exploratory}) in general reinforcement learning settings is suboptimal, where effective exploration is critical for policy improvement in stochastic MDPs \citep{zhang2024zeroth}. 
More recently, ZPG \citep{zhang2024zeroth} suggested an RLHF approach 
that does not rely on a reward model and is designed for non-deterministic MDPs. However, as the authors acknowledged, the algorithm lacks a strategic exploration mechanism. Furthermore, it relies on trajectory-level preference comparisons and performs on-policy updates, hence previously collected data are not reused. 

In this work, we introduce Dual-Feedback Actor (DFA), a reinforcement learning algorithm that works for stochastic MDPs and unifies scalar rewards and preference-based feedback into a single, principled policy update rule. Unlike many prior approaches in RLHF that infer a separate reward model from preferences, DFA directly incorporates preferences into the policy optimization objective using the policy’s log-probabilities and retains Soft Actor-Critic (SAC)-style entropy-driven exploration.  The main contributions are as follows:

\begin{itemize}

\item Our approach offers dual compatibility with both rewards and preferences. When numerical rewards are available, the agent updates its Q-networks and incorporates preference-based learning by synthesizing preferences from Q-values. This dual approach allows the agent to use reward signals while maintaining flexibility to incorporate human feedback, especially in settings where rewards are sparse or absent.

\item Our approach can be used not only in on-policy manner but also in off-policy manner, which enables more sample-efficient learning by reusing past experiences stored in a replay buffer. This is particularly valuable for hierarchical RL applications where sample efficiency is needed to train the policy of each layer.

\item Under the assumptions stated in Section~\ref{sec:theory}, we prove that minimizing DFA’s preference loss recovers the entropy-regularized SAC solution, formally bridging preference optimization and entropy-regularized RL. Consequently, DFA inherits SAC's entropy-driven exploration, maintaining diverse action sampling even when it learns solely from preferences.

\item Experimental results in Section~\ref{sec:experiments} show that DFA consistently matches or outperforms both reward-based and preference-based baselines on six control tasks and a stochastic GridWorld, while yielding a more stable training process.

\end{itemize}
\section{Related Work}
\label{sec:related-work}
There are two dominant paradigms for incorporating human feedback in reinforcement learning.
The first relies on reward modeling: These methods first fit a scalar reward (or value) prediction model from preference data and then treat this learned reward model as the surrogate reward for standard policy optimization.  This two-stage pipeline was introduced in \citep{christiano2017deep}, \citet{schoenauer2014programming}, and later scaled to large language models by \citet{ziegler2019fine}, \citet{stiennon2020learning}, and \citet{ouyang2022training}. 
The second relies on direct policy optimization: These algorithms bypass an explicit reward model and update the policy parameters solely from preference comparisons \citep{wilson2012bayesian, busa2014preference,akrour2011preference}. 

For reward-modeling approaches, several works \citep{saha2023dueling,zhu2023principled,wu2023making} consider linearly parameterized reward models and characterize the error bounds of the estimated parameters, and prove that subsequent reward-based RL can tolerate small errors in rewards. \citet{zhan2023provable} extend this analysis to more general reward function classes under some conditions. These analyses have been extended to direct policy optimization approaches in \citet{xu2020preference,chen2022human, zhang2024reinforcement}. 

In the context of language modeling, DPO \citep{rafailov2023direct} provides a direct approach to aligning language models with human preferences by optimizing a policy to maximize the likelihood of preferred responses over nonpreferred ones, eliminating the need for an explicit reward model. SPO \citep{sharifnassab2024soft} optimizes model output directly over a preference dataset through the natural conditional probability of the preferred responses over nonpreferred ones. 
{ Similar approaches have also been explored in this literature \citep{xu2024contrastive,ethayarajh2024kto, hong2024orpo,park2024disentangling, hong2024orpo, meng2024simpo,li2025length}}.
RLHF has also been studied in other aspects. 
For example, the framework in 
\citet{swamy2024minimaximalist} casts RLHF as a two-player zero-sum game. However, they still estimate the rewards (and subsequently apply PPO, TRPO, or SAC) based on a constantly updated queue of recent rollouts, which can cause data staleness issues.

Recent work, \citet{xie2024exploratory}, inspired by DPO, combines DPO with optimistic exploration to design XPO in the function approximation regime with provable convergence. 
ZPG \citep{zhang2024zeroth} aims to address RLHF without relying on a reward model and is designed for non-deterministic MDPs. However, it lacks an exploration mechanism, which is essential for general RL applications.
Although previous work brought advancement in several aspects, existing algorithms are rarely benchmarked (theoretically and experimentally) against strong reward-based baselines such as SAC. 

\section{Preliminaries}
\label{sec:preliminaries}

In this section, we introduce the notation for RL, RLHF, and review the DPO objective \citep{rafailov2023direct}. 
We model the environment as a finite–horizon Markov Decision Process (MDP).
An MDP can be represented as a tuple $\mathcal{M}=\langle\mathcal{S},\mathcal{A},P,R, \gamma,p_0\rangle$, where $\mathcal{S}$ and $\mathcal{A}$ are state space and action space, respectively. 
The conditional probability of transition from state $s$ to $s'$ with action $a$ is denoted by $P(s'|s,a)$. 
The probability distribution over the initial state $s_0$ is denoted by $p_0(s_0)$. 
The parameter $\gamma\in (0,1)$ denotes the discount factor.
At each time step $t$, $r(s_t,a_t)$ returns the reward of taking action $a_t$ in the state $s_t$. 
Actions are chosen according to the policy $\pi$ where $\pi(a|s)$ is the probability of taking action $a$ for a given state $s$. Here, we assume that the policy is parameterized with a vector $\theta\in \mathbb{R}^d$ and use shorthand notation $\pi_{\theta}$ for $\pi_{\theta}(a|s)$. For a given time horizon $H$, we define $\tau=(s_0,a_0,\cdots,s_{H-1},a_{H-1})$ as a sequence of state-action pairs called a trajectory. $R(\tau)$ is a function that returns the discounted accumulated reward of each trajectory as follows:
$R(\tau):=\sum_{h=0}^{H-1} \gamma^h r(s_h,a_h)$ where $\gamma\in (0,1)$ is the discount factor.

Given a policy $\pi$, the \emph{state-value function}
and the \emph{action-value function (or $Q$-function)} are
\[
V^{\pi}(s)
\;=\;
\mathbb{E}_{\tau\sim\pi}\!\Bigl[
        \sum_{t=0}^{H-1}\gamma^{t}\,r(s_t,a_t)
        \;\Big|\;s_0=s
\Bigr],
\]\[
Q^{\pi}(s,a)
\;=\;
\mathbb{E}_{\tau\sim\pi}\!\Bigl[
        \sum_{t=0}^{H-1}\gamma^{t}\,r(s_t,a_t)
        \;\Big|\;s_0=s,\;a_0=a
\Bigr].
\]

\paragraph{Classical RLHF feedback setting \citep{christiano2017deep}.}
Let $\mathcal{M}=\langle\mathcal{S},\mathcal{A},P,R,\gamma,p_0\rangle$ be the finite-horizon MDP where the true reward $r(s,a)$ is \emph{hidden}. 
Hence, we ask humans to compare trajectories and form the preference dataset 
\[
\D
=\bigl\{(\tau_k^{+},\tau_k^{-})\bigr\}_{k=1}^{K},
\qquad
\tau_k^{+}\succ\tau_k^{-}\;,
\]
where $K$ is the total number of pairs, and $\tau_k^+$ is \emph{preferred}
to $\tau_k^-$ which is denoted as $\tau_k^{+}\succ\tau_k^{-}$. 
We define a parametric function $r_\phi:\mathcal{S}\!\times\!\mathcal{A}\to\mathbb{R}$ to \emph{approximate} the latent reward.  
For any trajectory $\tau=(s_0,a_0,\dots,s_{H-1},a_{H-1})$, we define the model return as:
\[
R_\phi(\tau)=\sum_{h=0}^{H-1}\gamma^{h}\,r_\phi(s_h,a_h).
\]
The parameters~$\phi$ are learned by maximum likelihood under the Bradley–Terry model \citep{bradley1952rank}, which is equivalent to minimizing the following loss:
\[
\mathcal{L}(\phi)=
-\;\mathbb{E}_{(\tau^{+},\tau^{-})\sim\D}
\bigl[
\log \sigma\!\bigl(R_\phi(\tau^{+})-R_\phi(\tau^{-})\bigr)
\bigr],
\]\[
\sigma(z)=\frac{1}{1+e^{-z}}.
\]
Let 
$\hat r_\phi$ be the estimated reward function. Next, 
with $\hat r_\phi$ fixed, a policy-gradient method such as PPO \citep{schulman2017proximal} or SAC \citep{haarnoja2018soft} updates $\pi_\theta$ to maximize
\[
J(\theta)=
\mathbb{E}_{\tau\sim\pi_\theta}\!\bigl[\hat R_{\phi}(\tau)\bigr].
\]

A well-known drawback on this two-stage pipeline is its sensitivity to noise, and any overfitting in $\hat r_\phi$ propagates directly to the final policy updates \citep{casper2023open}.

\paragraph{DPO in language models \citep{rafailov2023direct}}
In language models, a \textit{state} is the text prefix
(or prompt) $x$, and an \textit{action} is the response $y$ produced
by the model (call it continuations).  
Annotators make a choice among two full continuations $(y^{+},y^{-})$ sampled from the
same prompt, giving the preference dataset
\[
\D
=\bigl\{(x_k,y_k^{+},y_k^{-})\bigr\}_{k=1}^{K},
\qquad
y_k^{+}\succ y_k^{-}.
\]
Let $\pi_{\text{ref}}$ be the frozen base model (e.g.\ a pre-trained
GPT checkpoint).  
For a prompt $x$ and two candidate continuations
$y^{+},y^{-}$, define the log-probability gap:
\[
\Delta_{x,y^{+},y^{-}}(\theta)=
\bigl[\log\pi_\theta(y^{+}\!\mid\!x)-\log\pi_\theta(y^{-}\!\mid\!x)\bigr]
-\]\[\qquad \qquad \qquad \bigl[\log\pi_{\text{ref}}(y^{+}\!\mid\!x)-\log\pi_{\text{ref}}(y^{-}\!\mid\!x)\bigr].
\]
$\Delta_{x,y^{+},y^{-}}(\theta)$ captures how much more the new model
prefers the chosen continuation over the rejected one, \emph{relative}
to the base model. DPO then minimizes
\[
\mathcal{L}_{\text{DPO}}(\theta)=
-\;\mathbb{E}_{(x,y^{+},y^{-})\sim\D}
\Bigl[
\log\sigma\!\bigl(\alpha\,\Delta_{x,y^{+},y^{-}}(\theta)\bigr)
\Bigr],
\]\[
\sigma(z)=\tfrac{1}{1+e^{-z}},\;\alpha>0.
\]
  
Minimizing $\mathcal{L}_{\text{DPO}}$ pushes the new model
toward the preferred continuation, while limiting it to the safe
behavior of $\pi_{\text{ref}}$.
The absence of a separate reward model in DPO removes a major source of overfitting or noisy evaluations of the reward modeling. 
However, DPO assumes a Bradley-Terry choice model to derive its loss function, and this loss tends to produce near-deterministic models. This reduced diversity makes DPO prone to mode collapse \citep{azar2024general}.


\section{Methods}
\label{sec:methodologies}

In this section, we introduce our \emph{Dual-Feedback Actor} (DFA).  In order
to describe the DFA algorithm, we first introduce the state-wise feedback setting as follows:

\paragraph{State-wise feedback.} 
In this setting, the annotator does \emph{not} compare full trajectories.  
Instead, at a given state $s_k$, the annotator sees two actions, marks the
winner $a_k^{+}$ over the loser $a_k^{-}$. Then, the following preference dataset is formed:
\[
\D\;=\;\bigl\{(s_k,a_k^{+},a_k^{-})\bigr\}_{k=1}^{K},
\qquad a_k^{+}\succ a_k^{-},
\]
where $a^+$ is \emph{preferred}
to $a^-$ at state $s_k$. In the subsections below, we first consider the case where the agent learns only from state-wise human
comparisons. Second, we show how to synthesize preferences from numerical rewards when they are available.  Finally, we extend DFA to trajectory-based comparisons. 
\subsection{Learning with Only State‐wise Preferences}
\label{sec:dfa-statewise}

Assume we have collected a set of preference comparisons  
\[
\D
\;=\;
\bigl\{(s_k,a_k^{+},a_k^{-})\bigr\}_{k=1}^{K},
\qquad
a_k^{+}\succ a_k^{-}\;.
\]
Unlike classical RLHF, \emph{we do not assume an underlying
Bradley–Terry reward model}. 
Instead, we rely on the policy’s log-probabilities to model the preference probability directly.
For any pair $(s,a^{+},a^{-})$ we define the \emph{preference
probability} produced by the current policy $\pi_\theta$ as
\begin{equation}
\label{eq:cond-prob}
P_\theta\!\bigl(a^{+}\succ a^{-}\mid s\bigr)
=\frac{\pi_\theta(a^{+}\!\mid\!s)^{\alpha}}
       {\pi_\theta(a^{+}\!\mid\!s)^{\alpha}
        +\pi_\theta(a^{-}\!\mid\!s)^{\alpha}},
\qquad
\alpha>0.
\end{equation}
The exponent $\alpha$ controls the uncertainty assigned to the policy's output:
$\alpha\!\rightarrow\!0$ yields a nearly uniform (high‐entropy) choice,
while $\alpha\!\rightarrow\!\infty$ approaches a hard winner–takes–all
rule.

The negative log-likelihood \eqref{eq:cond-prob} gives the state-wise preference loss:
\begin{equation}
\mathcal{L}_{\text{pref}}(\theta)
\;=\;
-\;
\E_{(s,a^{+},a^{-})\sim\D}
\!\left[
\log
P_\theta\!\bigl(a^{+}\succ a^{-}\mid s\bigr)
\right].
    \label{eq:pref-loss}
\end{equation}
Minimizing $\mathcal{L}_{\text{pref}}$ directly increases the
probability that $\pi_\theta$ selects the human-preferred action,
without introducing auxiliary reward networks or relying on any latent
utility assumptions \footnote{%
Eq.~\eqref{eq:pref-loss} is identical 
to the \emph{preference loss} \(\mathcal{L}^{\alpha}_{\mathrm{pref}}\) used in Soft
Preference Optimization (SPO) \citep{sharifnassab2024soft}. Although DFA adopts the same logistic pairwise-loss form, the similarity ends there.
In SPO, the same term is combined with a global KL regularizer
\(\mathrm{D}_{\text{KL}}(\pi_\theta\!\parallel\!\pi_{\text{ref}})\),
whereas here we study the stand-alone preference part and show that,
under some assumptions, it  aligns the policy with the
entropy–regularized RL solution (Theorem \ref{thm:pref-sat}). 
Moreover,  SPO is in the context of LLMs and is designed for an offline setting. DFA targets stochastic MDPs, supports off-policy replay, preserves SAC-style entropy exploration with theoretical analysis, and unifies numeric rewards with preferences. Synthesizing preferences, as explained in Section \ref{sec:synth-pref-rewards}, is another key innovation in DFA that allows for online settings in RL.}.
Note that \eqref{eq:pref-loss} can be reformulated as follows:

\begin{align}
\mathcal{L}_{\text{pref}}(\theta) =\;
    & -\E_{( s,a^{+},a^{-})\sim \D} \Big[ 
        \log \sigma\Big( \alpha \big( \log \pi_\theta(a^+|{s}) \notag \\
    & \qquad\qquad\qquad\qquad - \log \pi_\theta(a^-|{s}) \big) \Big) 
    \Big] \notag
\end{align}

 In simulated environments or settings where numerical rewards are accessible, it is possible to synthesize preference data from these rewards or their proxies, such as Q-values. Our approach, introduced in the next section, is particularly useful when integrating preference-based learning into an agent's training loop, even when direct human feedback is unavailable or insufficient. Our method fuses numerical rewards and preference data by synthesizing preferences from numerical rewards.

\subsection{Synthesizing Preferences from Numerical Rewards}
\label{sec:synth-pref-rewards}

We use Q-values as a proxy to create preference pairs, enabling online preference generation during policy updates without explicitly constructing full trajectory segments. Estimating Q-values can be done through any method in the literature, and is particularly relevant in off-policy methods such as SAC, where a replay buffer stores past experiences as tuples $(s_t, a_t, r_t)$, where $s_t, a_t,$ and $r_t$ are state, action and reward at time $t$, respectively. 

Our approach works as follows:
For a batch of states $\{s_i\}_{i=1}^N$ sampled from the replay buffer, we generate two candidate actions to form preference pairs:
The first action, denoted by $a_i$, corresponds to the action originally taken in state $s_i$ as stored in the replay buffer. 
This action reflects the historical behavior of the agent at the time the state was visited. 
The second action, denoted by $a'_i$, is obtained from the replay buffer by identifying the action associated with the nearest state to $s_i$ (denote it with $s'_i$)\footnote{In our experiments, we compute the Euclidean distance between $s_i$ and all states in the buffer, select the closest state $s'_i$, and retrieve its corresponding action $a'_i$.}. 
     For both actions, we compute their respective Q-values.
     The action with the higher Q-value is designated as the preferred action $a_i^+$, while the other is labeled as the rejected action $a_i^-$:
    \[
    \text{If } Q(s_i, a_i) > Q(s_i, a'_i), \qquad \text{ then } (a_i^+, a_i^-) = (a_i, a'_i), \qquad \]
    \[\text{else } (a_i^+, a_i^-) = (a'_i, a_i).
    \]

This process effectively synthesizes preference data in the form of state-action pairs $\D^{\text{Syn}}=\{({s}_i, a_i^+, a_i^-)\}_{i=1}^N$ for each batch. The loss in \eqref{eq:pref-loss}, is then calculated over the states $\{{s}_i\}_{i=1}^N$ and using their associated preferred and rejected actions. Specifically, the loss encourages the policy to assign higher probability to preferred actions over rejected ones, scaled by the parameter $\alpha$

\begin{align}
\loss^{\text{Syn}}(\theta) = 
    & -\E_{( s_i,a_i^{+},a_i^{-})\sim \D^{\text{Syn}}} \Big[ 
    \log \Big( \sigma\big( \alpha \big( \log \pi_\theta(a_i^+|{s}_i) \notag \\
    & \qquad - \log \pi_\theta(a_i^-|{s}_i) \big) \big) \Big) 
    \Big]\notag
\end{align}
where $\sigma(\cdot)$ is the sigmoid function. Figure \ref{fig:dfa_workflow} gives a high-level schematic of our methodology.

\begin{figure}[htp!]
  \centering
  \resizebox{0.99\linewidth}{!}{%
  \begin{tikzpicture}[
      font=\small, >=latex,
      box/.style  ={draw, rounded corners=4pt, minimum width=3.7cm,
                    minimum height=1.7cm, text width=3.5cm,
                    align=center, font=\scshape},
      agent/.style ={box, fill=blue!10},
      buffer/.style={box, fill=orange!12},
      synth/.style ={box, fill=green!12},
      policy/.style={box, fill=purple!10},
      arrow/.style ={->, thick}
    ]

    \node[agent]  (agent)            {Agent\\$\pi_\theta$};
    \node[buffer, right=4cm of agent]    (buffer){Rollout\\/Replay Buffer};
    \node[policy, right=3cm of buffer]   (policy){Policy Update\\(min.\ $\mathcal{L}_{\text{pref}}$)};

    \coordinate (mid) at ($(buffer)!0.5!(policy)$);
    \node[synth] (synth) at ($(mid)+(0,3)$) {Preference\\Synthesizer};

    \draw[arrow] (agent) -- (buffer) 
        node[midway, above, font=\footnotesize]{Interact with environment};

    \draw[arrow] (buffer) -- (policy)
        node[midway, above, font=\footnotesize]{Human preferences};

    \draw[arrow] (buffer.north) -- (synth.south)
        node[midway, left, font=\footnotesize]{$(s,a,r)$};

    \draw[arrow] (synth.south) -- (policy.north)
        node[midway, right, font=\footnotesize]{  Synthetic preferences};

    \draw[arrow] (policy.south) .. controls +(down:1.5)  ..
        (agent.south east)
        node[midway, below, font=\footnotesize]{updated $\pi_\theta$};

  \end{tikzpicture}
  }
\caption{%
Data flow in DFA. The agent executes its current policy $\pi_\theta$ and stores the transitions (may include reward-based transitions or human-annotated preferences).  If reward-based transitions are available, the \emph{Preference Synthesizer} can convert them into synthetic preference pairs. This process can be done in either an on-policy or off-policy fashion.
Both human and synthetic preferences can be used in \emph{Policy-Update}, which minimizes the preference loss $\loss$ and outputs an improved policy.%
}
  \label{fig:dfa_workflow}
\end{figure}
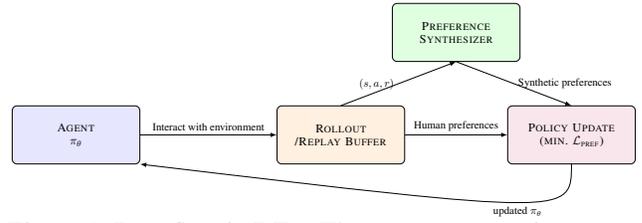

\subsection{Extension of the Loss to Trajectory-based Comparisons}
\label{sec:method_traj}

State-wise comparisons can be easy to collect (for instance, a single frame rather than a full video in video games) and give richer training signals, but one may prefer to rank the whole trajectories \citep{christiano2017deep, zhang2024zeroth}, hence, we extend DFA to accept trajectory-level preferences as well.
For trajectory-level comparisons, we store pairs  
\[
\D^{\text{traj}}
=\bigl\{(\tau_k^{+},\tau_k^{-})\bigr\}_{k=1}^{K},
\qquad
\tau=(s_1,a_1,\dots,s_T).
\]

The policy assigns a likelihood to any full trajectory as follows: $\pi_\theta(\tau)\;=\;\prod_{t=1}^{T}\pi_\theta\!\bigl(a_t\mid s_t\bigr)$.
The preference probability is the same as before, but now in terms of trajectory likelihoods:

\begin{equation}
\label{eq:cond-prob-traj}
P^{\text{traj}}_\theta\!\bigl(\tau^{+}\succ \tau^{-}\bigr)
=\frac{\pi_\theta(\tau^{+})^{\alpha}}
       {\pi_\theta(\tau^{+})^{\alpha}
        +\pi_\theta(\tau^{-})^{\alpha}},
\qquad
\alpha>0.
\end{equation}
The negative log-likelihood of \eqref{eq:cond-prob-traj} gives trajectory-based preference loss:
\begin{equation}
\loss^{\text{traj}}(\theta)
\;=\;
-\;
\E_{(\tau^{+},\tau^{-})\sim\D^{\text{traj}}}
\!\left[
\log
P^{\text{traj}}_\theta\!\bigl(\tau^{+}\succ \tau^{-}\bigr)
\right].
    \label{eq:pref-loss-traj}
\end{equation}

\section{Theoretical Analysis}
\label{sec:theory}

In this section, we show that, under
Bradley–Terry model on the soft optimal $Q$–function, minimizing
our preference loss is equivalent to recovering
the optimal policy for entropy–regularized reinforcement
learning \citep{haarnoja2017soft}.  
Concretely, we analyze the tabular setting for the state-wise preferences and identify its unique minimizer. This establishes the equivalence of preference optimization and entropy-regularized RL. 
{We should emphasize that the BT model is not a requirement of DFA Algorithm; it is only used to derive Theorem \ref{thm:pref-sat} in the following. A trajectory-wise analysis and its connection to the state-wise analysis, is provided in Appendix \ref{sec:traj-analysis}.}

\begin{assumption}[Bradley--Terry preferences on the \emph{soft}-optimal \(Q\)-function]
\label{as:BT-Qstar}
Let \(Q^{\star} : \mathcal{S}\!\times\!\mathcal{A}\!\to\!\mathbb{R}\) be the soft-optimal
state-action value function of the MDP, i.e.,

\begin{align}
Q^{\star}(s,a)\;=\;
      &\max_{\pi}\;
      \E\!\Bigl[\;
          \sum_{t=0}^{\infty}\gamma^{t}
          \bigl(r(s_t,a_t)\;+\;\lambda\,
                 \mathcal{H}\!\bigl(\pi(\,\cdot\!\mid\!s_t)\bigr)\bigr) \notag \\
      &\qquad\qquad\qquad\;\Big|\;s_{0}=s,\;a_{0}=a
      \Bigr],\notag 
\end{align}
where $\mathcal{H}(.)$ is the entropy function and $\lambda$ is the entropy coefficient.
Assume that there exists a parameter \(\beta>0\) such that, for every
\(s\in\mathcal{S}\) and any \(a,b\in\mathcal{A}\),
\[
P^{\star}\!\bigl(a\succ b \mid s\bigr)
  \;=\;
  \sigma\!\Bigl(
      \beta\,[\,Q^{\star}(s,a)-Q^{\star}(s,b)\,]
  \Bigr),
  \]\[
  \sigma(z)=\frac{1}{1+e^{-z}}.
\]
\end{assumption}

\begin{theorem}[Preference loss recovers the optimal
policy]
\label{thm:pref-sat}
Fix a state \(s\in\mathcal{S}\) and abbreviate
\(Q^{\star}_a \coloneqq Q^{\star}(s,a)\).
Suppose that \Cref{as:BT-Qstar} holds.
Under uniform sampling of \emph{ordered} pairs
\((a,b)\sim\mathrm{Unif}(\mathcal{A}^{2})\)
and the tabular full-support
parameterization
\(
  \ell_a=\log\pi(a\mid s)
\)
(\(\sum_a e^{\ell_a}=1,\;e^{\ell_a}>0\)),
consider the preference loss
\begin{align}
\mathcal{L}(\boldsymbol{\ell})
  \;=&&\;
  -\frac{1}{|\mathcal{A}|^{2}}
   \sum_{(a,b)\in\mathcal{A}}
   P^{\star}(a\succ b\mid s)\,
   \log\sigma\!\bigl(\alpha(\ell_a-\ell_b)\bigr),\nonumber
  \\
  &&\alpha>0. 
\end{align}
This loss is strictly convex on the set of full-support policies and is
minimized uniquely at
\begin{equation}
\pi_{\star}(a\mid s)
  \;=\;
  \frac{\exp\!\bigl(\tfrac{\beta}{\alpha}\,Q^{\star}_a\bigr)}
       {\displaystyle\sum_{a'\in\mathcal{A}}
        \exp\!\bigl(\tfrac{\beta}{\alpha}\,Q^{\star}_{a'}\bigr)}.
\label{eq:pi-minimizer}  
\end{equation}

Furthermore, this Gibbs distribution coincides with the \emph{global} maximizer of the entropy-regularized RL (or SAC objective) when \(\lambda=\alpha/\beta\):\footnote{All proofs are provided in the Appendix.}
        \begin{equation}
        \max_{\pi}\;
        \E_{\pi}\!\Bigl[
          \sum_{t=0}^{\infty}\gamma^{t}
          \bigl(r(s_t,a_t)+\lambda\,\mathcal{H}(\pi(\cdot\mid s_t))\bigr)
        \Bigr].
        \label{eq:ent-reg}            
        \end{equation}

\end{theorem}

Theorem~\ref{thm:pref-sat} states that, when human (or synthetic) comparisons follow a
Bradley-Terry model whose latent utility equals the ground truth
$Q^{\star}$, the preference loss is perfectly aligned with the
entropy-regularized control objective
\citep{haarnoja2017soft}.  
The optimizer
\eqref{eq:pi-minimizer} is soft-max policy whose
inverse temperature is the ratio $\beta/\alpha$: The parameter
$\beta$ captures how consistently the annotator prefers higher-value
actions, while the parameter $\alpha$ adjusts the learner's
uncertainty.  
In particular, setting $\lambda=\alpha/\beta$ recovers the SAC
trade-off between exploitation (large $\beta$) and exploration (large
$\alpha$) \citep{haarnoja2018soft}.

\begin{remark}
\label{rem:sac}
If the Bradley-Terry assumption holds for any \emph{arbitrary} soft
state-action value function, for instance, the current
critic estimate $Q_k(s,a)$ in SAC \citep{haarnoja2018soft}.  Then Theorem~\ref{thm:pref-sat} implies
that the preference loss is minimized by  
\[
\pi_{k+1}(a \mid s)=
\frac{\exp\!\bigl(\tfrac{\beta}{\alpha}\,Q_k(s,a)\bigr)}
     {\displaystyle\sum_{a'} \exp\!\bigl(\tfrac{\beta}{\alpha}\, Q_k(s,a')\bigr)}.
\]
This update is \emph{exactly} the policy-improvement step in SAC that maximizes the
entropy-regularized objective
\[
  \max_{\pi}\;
  \Bigl\{\,\mathbb{E}_{a\sim\pi}\bigl[Q_k(s,a)\bigr]
          +\lambda\,\mathcal H\!\bigl(\pi(\cdot\mid s)\bigr)
  \Bigr\}.
\]
Hence, as the critic converges ($Q_k\!\to\!Q^{\star}$), repeated
minimization of the preference loss yields the soft-optimal SAC policy. Therefore, preference learning can be
viewed as performing policy improvement in 
SAC, but driven solely by comparative feedback.
\end{remark}

\begin{remark}
\label{rem:nonuniform}
The assumption that \((a,b) \sim \mathrm{Unif}(\mathcal{A}^2)\) in Theorem~\ref{thm:pref-sat} is made for simplicity of analysis. 
In practice, one can approximate this condition by drawing a mini-batch of
ordered pairs at each update and down-sampling (or re-weighting) each
pair by the inverse of its frequency in the batch;
This produces a uniform
sub-sampled action pairs required by the theorem.
\end{remark}

\section{Experimental Results}
\label{sec:experiments}

In this section, we benchmark DFA against prior work. 
The complete code is provided in the Supplemental Material. 
We first compare DFA with the reward-based baseline SAC (hence, we have to synthesize preferences following Section \ref{sec:synth-pref-rewards}), and then against recent preference-based methods.

\subsection{Comparison with SAC via Synthetic Preferences}
\label{sec:exp-sac-synth}

\begin{figure*}[htp!]
    \centering
    \includegraphics[width=0.99\textwidth]{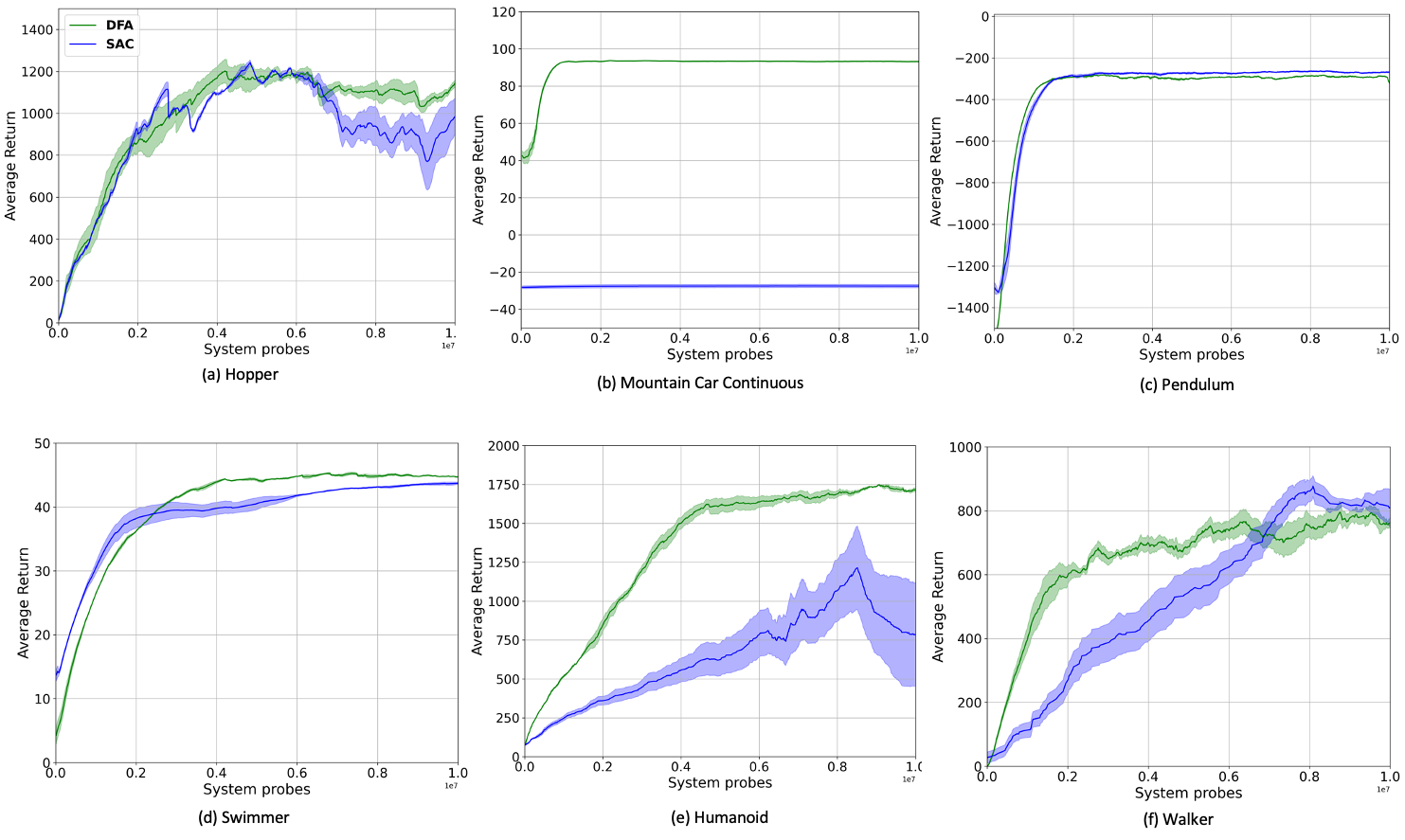}
    \caption{DFA (green) vs.\ SAC (blue) on the six MuJoCo control tasks.  
     DFA matches or exceeds SAC and shows smoother training. The solid line is the mean episode return, and the shaded region shows an 90\% confidence interval over 5 seeds.}

    \label{fig:sac_synth_curves}
\end{figure*}
In this section, we evaluate the proposed algorithm (DFA) and compare it with related work on six control tasks in MuJoCo~\citep{todorov2012mujoco}, a physics simulator known for fast and accurate simulations in areas such as robotics, biomechanics, and graphics. Since published benchmarks (e.g., OpenAI SpinningUp at \textit{https://spinningup.openai.com/en/latest/spinningup/bench .html\#benchmarks-for-spinning-up-implementations}) consistently identify SAC as the strongest baseline on many environments, we compare DFA exclusively with SAC.
We briefly explain the six environments we consider in Appendix \ref{apdx:hyperparameters}.


In Figure \ref{fig:sac_synth_curves}, we monitor the average episode return versus system probes, which represents the total number of environment interactions. In this experiment, DFA continually generates synthetic preference pairs from numerical rewards following \ref{sec:synth-pref-rewards}. The underlying RL settings and replay buffer are identical to those of SAC.
Both DFA and SAC run for \(10\!\times\!10^{6}\) system probes across 5 different random seeds. We use mini-batch size \(256\) for both algorithms. 
A new preference batch of size $N\!=\!256$ is created during every gradient step.
Figure~\ref{fig:sac_synth_curves} shows that DFA matches or exceeds SAC on Walker2d, Hopper, Swimmer, and Humanoid. 
In MountainCarContinuous, we could not find SAC settings that produced learning, a problem others have reported as well.\footnote{\url{https://github.com/rail-berkeley/softlearning/issues/76}}  
DFA, in contrast, learned a good policy on this task with the same range of hyperparameters used for the other environments.

Interestingly, DFA’s learning curves are noticeably smoother, while SAC exhibits significant fluctuations. 
We attribute this stability to the synthesized preference pairs, which are constructed according to Section \ref{sec:synth-pref-rewards}, and appear to act as an implicit denoising regularizer.
We note that reducing the learning rate or adjusting other hyperparameters to avoid fluctuations for SAC resulted in lower average returns, thus, we maintained the higher learning rate configuration to ensure fair comparison.

These results confirm the claim of Theorem \ref{thm:pref-sat}: \emph{once we use preference data aligned with the optimal Q-values, numerical rewards can be dropped without losing performance}. This unifies reward-free human alignment and reward-based RL under a single log-likelihood objective.

\subsection{Comparison with RM Methods}
\label{sec:exp-rm}

\begin{figure*}[htp!]
    \centering
    \begin{subfigure}[t]{0.48\textwidth}
        \centering
        \includegraphics[width=\textwidth]{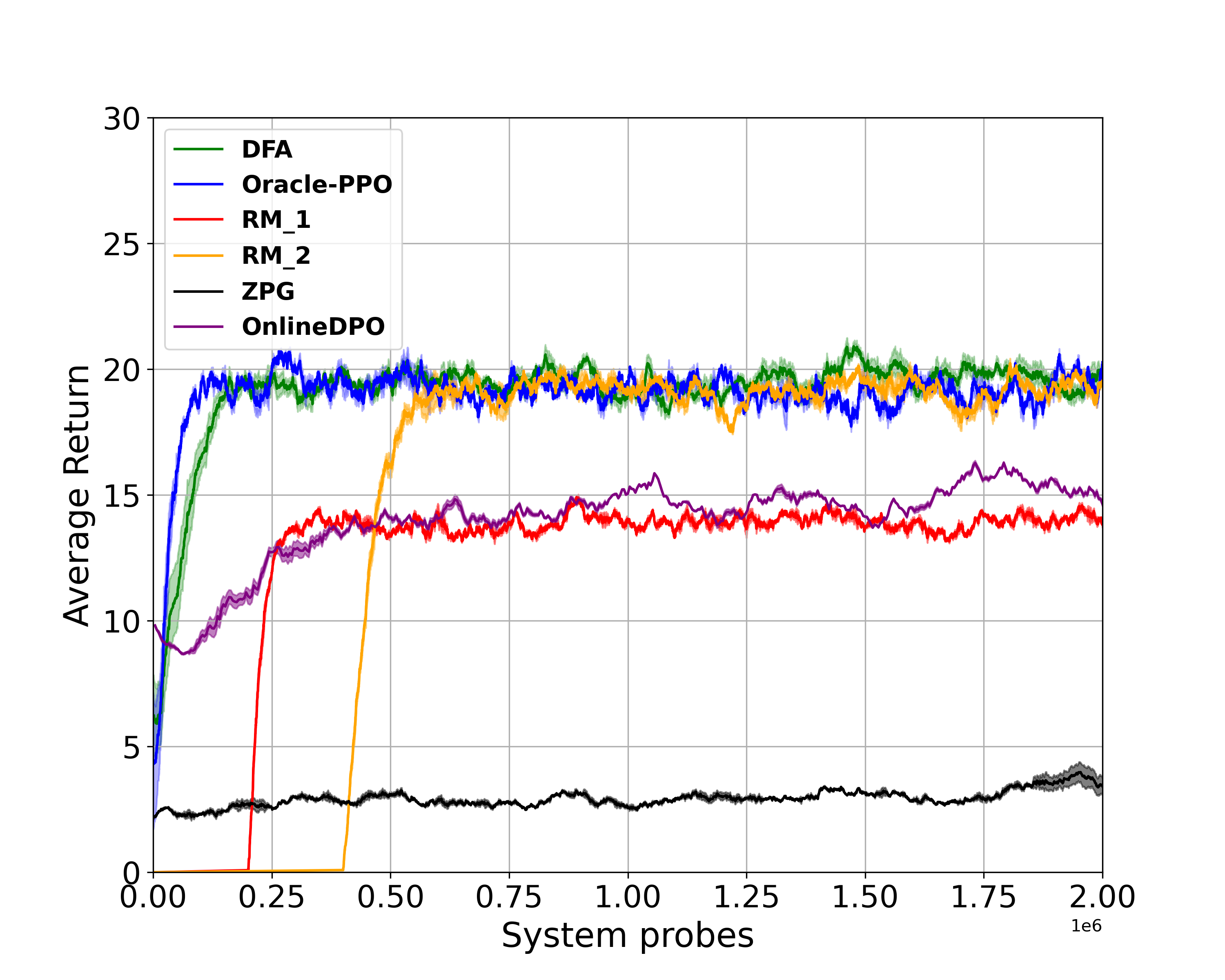}
        \caption{DFA vs.\ RM\,+\,PPO and oracle PPO.}
        \label{fig:rm_curves}
    \end{subfigure}
    \hfill
    \begin{subfigure}[t]{0.48\textwidth}
        \centering
        \includegraphics[width=\textwidth]{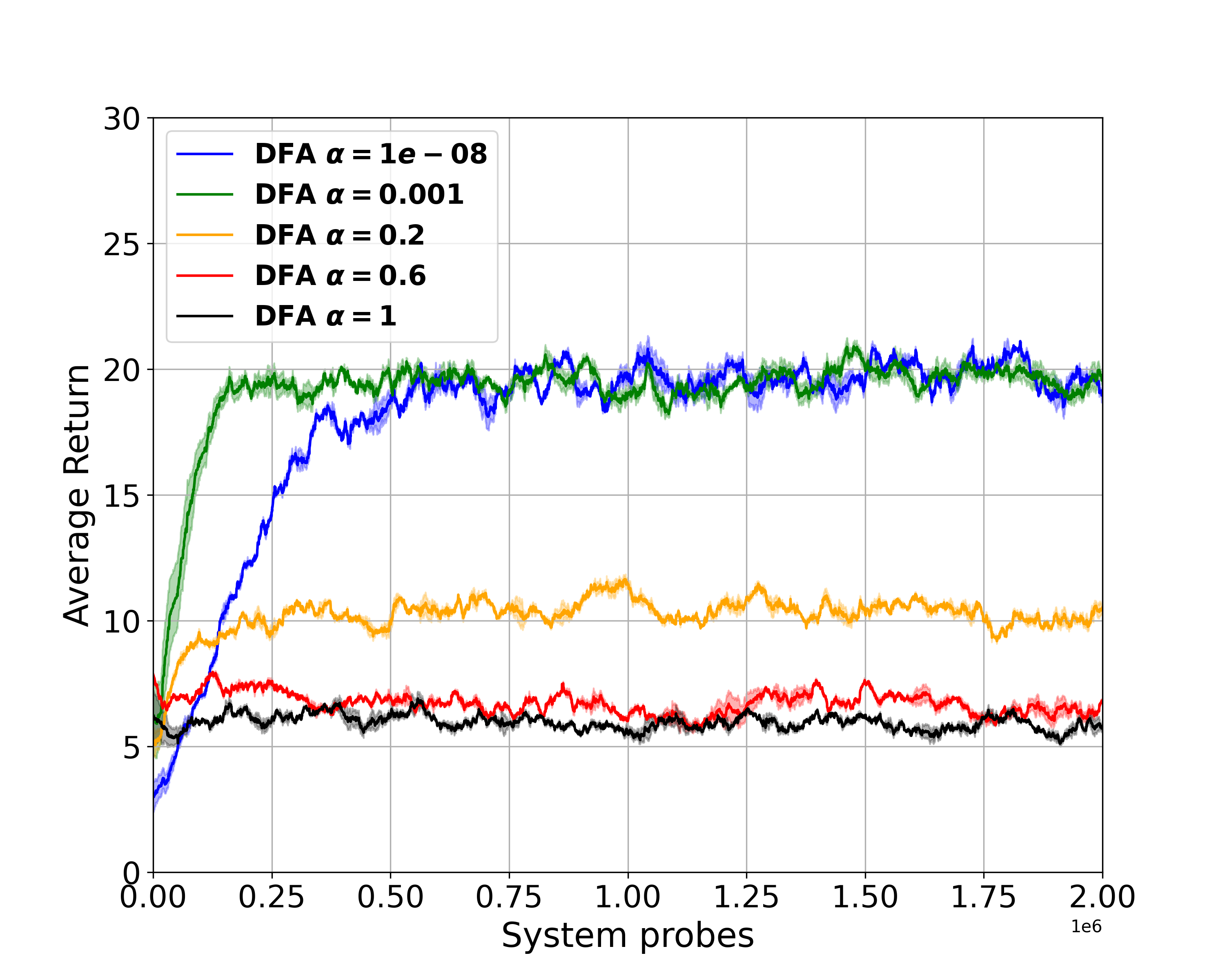}
        \caption{Effect of temperature $\alpha$.}
        \label{fig:alpha}
    \end{subfigure}
    \caption{GridWorld results. (a)~DFA learns faster and achieves higher rewards than reward-modeling baselines, approaching the oracle that has access to the true reward.  
    (b)~Effect of the temperature parameter~$\alpha$: a small but not too small value balances exploration and exploitation.  Shaded regions denote 90\% confidence intervals across 5 random seeds.}
    \label{fig:side_by_side}
\end{figure*}
In this section, we evaluate our DFA algorithm against traditional reward modeling approaches in the context of learning from human preferences. While the previous section demonstrated DFA's effectiveness with generated preferences derived from numerical rewards, here we focus on the more challenging scenario where only human comparative feedback is available, without access to ground-truth rewards.

We conduct experiments in a stochastic GridWorld environment, which provides a controlled testbed for preference-based learning \citep{zhang2024zeroth}. In this environment, the agent starts at the center of the grid and can take four actions: up, down, left, or right. The environment includes the following aspects: (1) To build the ground play, a coin is flipped for each cell, and if heads, a reward sampled from $\mathcal{N}(0,1)$ is placed in that cell; (2) While the agent is moving, with probability $0.4$, the chosen action is reversed (e.g., "up" becomes "down"). Each episode has a fixed horizon of 20 steps, and the agent's goal is to maximize the cumulative reward collected. This environment is particularly suitable for preference-based learning evaluation as it combines stochastic dynamics with a non-trivial reward structure that requires exploration.

To simulate human preferences, we simulate a panel of annotators who provide comparative feedback between trajectories. Following standard practice in RLHF literature, we model the annotator's preference probability using the Bradley-Terry model: $P(\tau_1 \succ \tau_0) = \sigma(R_1 - R_0)$, where $R_i$ is the cumulative reward of trajectory $\tau_i$ and $\sigma$ is the sigmoid function. For robustness, each preference query aggregates votes from $M$ independent annotators, and the majority vote determines the final preference. This approach simulates the noise and variability in real human feedback while maintaining a consistent underlying reward structure. For more implementation details, please see Appendix \ref{apdx:hyperparameters}.
We compare DFA against the following approaches:
\begin{itemize}
    \item \textbf{RM+PPO}: A two-stage approach that first learns a reward model from preference data using maximum likelihood estimation, then optimizes a policy using Proximal Policy Optimization (PPO) with the learned reward function.
    \item \textbf{ZPG \citep{zhang2024zeroth}}: A state-of-the-art RLHF method which estimates the policy gradient from preference differences without fitting a reward model.
    \item \textbf{Oracle-PPO (upper bound)}:
          PPO directly on the \emph{true} MDP reward $r$ (it is unavailable in
          practice, but gives an upper bound on the performance.).
    \item \textbf{OnlineDPO}: We also include the recently-proposed OnlineDPO algorithm \citep{guo2024direct} as a direct-preference baseline.
\end{itemize}

Figure \ref{fig:rm_curves} demonstrates that DFA consistently outperforms reward modeling methods and performs comparably to Oracle-PPO, which has access to the true reward function. In this experiment, we compare against two variants of RM+PPO: RM\_1 uses 200k environment steps for training the reward model, while RM\_2 uses twice as many samples (400k steps). Despite the increased data budget for RM\_2, DFA is still converging faster, highlighting the benefits of avoiding the two-stage pipeline.  
For the implementation of ZPG, we contacted the authors for the official implementation, but they indicated that the code is undergoing intellectual review.  Consequently, we re-implemented the algorithm from the paper, closely matching hyperparameters and implementation details.  
Despite our efforts (and implementation tricks such as normalized gradient and gradient clipping), ZPG could not be tuned to outperform the results shown in Figure \ref{fig:rm_curves}; we therefore report its best observed performance.  In 
Figure \ref{fig:rm_curves} we use an annotator pool of $M\!=\!500$; runs with smaller $M$ and more complex environments show the same pattern and are included in Appendix~\ref{apdx:hyperparameters}.

Figure~\ref{fig:alpha} highlights the sensitivity of DFA to the parameter $\alpha$.  
As shown in Figure~\ref{fig:alpha}, setting $\alpha$ too high (\,$\alpha=1.0$\,) gives almost no learning signal, while moderate values in the range $0.2\!-\!0.6$ yield better results. The best result comes at $\alpha=0.001$.  
When $\alpha$ is pushed to very small values (e.g., $10^{-8}$), performance drops again because the policy becomes overly stochastic.  
These results suggest that $\alpha$ should be small but not too small to balance exploration and exploitation.

\section{Conclusion}
\label{sec:conclusion}

 Dual-Feedback Actor (DFA) unifies scalar rewards and pairwise preferences in a single loss; when preferences follow a Bradley–Terry model on the optimal soft $Q$-function, this loss recovers the entropy-regularized SAC solution, formally linking reward- and preference-based RL. Empirically, DFA matches or exceeds SAC and outperforms reward-modeling baselines while training more smoothly. The main limitations are the Bradley–Terry assumption, the noise inherited by synthetic preferences from early $Q$ estimates, and the computational cost of finding the nearest state in the replay buffer. The future work can be investigating other assumptions and evaluating DFA on larger, real human-in-the-loop tasks.

\bibliography{aaai2026}

\begin{thebibliography}{44}
\providecommand{\natexlab}[1]{#1}

\bibitem[{Absil, Mahony, and Sepulchre(2009)}]{absil2009optimization}
Absil, P.-A.; Mahony, R.; and Sepulchre, R. 2009.
\newblock Optimization algorithms on matrix manifolds.
\newblock In \emph{Optimization Algorithms on Matrix Manifolds}. Princeton University Press.

\bibitem[{Akrour, Schoenauer, and Sebag(2011)}]{akrour2011preference}
Akrour, R.; Schoenauer, M.; and Sebag, M. 2011.
\newblock Preference-based policy learning.
\newblock In \emph{Machine Learning and Knowledge Discovery in Databases: European Conference, ECML PKDD 2011, Athens, Greece, September 5-9, 2011. Proceedings, Part I 11}, 12--27. Springer.

\bibitem[{Azar et~al.(2024)Azar, Guo, Piot, Munos, Rowland, Valko, and Calandriello}]{azar2024general}
Azar, M.~G.; Guo, Z.~D.; Piot, B.; Munos, R.; Rowland, M.; Valko, M.; and Calandriello, D. 2024.
\newblock A general theoretical paradigm to understand learning from human preferences.
\newblock In \emph{International Conference on Artificial Intelligence and Statistics}, 4447--4455. PMLR.

\bibitem[{Bradley and Terry(1952)}]{bradley1952rank}
Bradley, R.~A.; and Terry, M.~E. 1952.
\newblock Rank analysis of incomplete block designs: I. The method of paired comparisons.
\newblock \emph{Biometrika}, 39(3/4): 324--345.

\bibitem[{Busa-Fekete et~al.(2014)Busa-Fekete, Sz{\"o}r{\'e}nyi, Weng, Cheng, and H{\"u}llermeier}]{busa2014preference}
Busa-Fekete, R.; Sz{\"o}r{\'e}nyi, B.; Weng, P.; Cheng, W.; and H{\"u}llermeier, E. 2014.
\newblock Preference-based reinforcement learning: evolutionary direct policy search using a preference-based racing algorithm.
\newblock \emph{Machine learning}, 97: 327--351.

\bibitem[{Casper et~al.(2023)Casper, Davies, Shi, Gilbert, Scheurer, Rando, Freedman, Korbak, Lindner, Freire et~al.}]{casper2023open}
Casper, S.; Davies, X.; Shi, C.; Gilbert, T.~K.; Scheurer, J.; Rando, J.; Freedman, R.; Korbak, T.; Lindner, D.; Freire, P.; et~al. 2023.
\newblock Open problems and fundamental limitations of reinforcement learning from human feedback.
\newblock \emph{arXiv preprint arXiv:2307.15217}.

\bibitem[{Chen et~al.(2022)Chen, Zhong, Yang, Wang, and Wang}]{chen2022human}
Chen, X.; Zhong, H.; Yang, Z.; Wang, Z.; and Wang, L. 2022.
\newblock Human-in-the-loop: Provably efficient preference-based reinforcement learning with general function approximation.
\newblock In \emph{International Conference on Machine Learning}, 3773--3793. PMLR.

\bibitem[{Christiano et~al.(2017)Christiano, Leike, Brown, Martic, Legg, and Amodei}]{christiano2017deep}
Christiano, P.~F.; Leike, J.; Brown, T.; Martic, M.; Legg, S.; and Amodei, D. 2017.
\newblock Deep reinforcement learning from human preferences.
\newblock \emph{Advances in neural information processing systems}, 30.

\bibitem[{Ethayarajh et~al.(2024)Ethayarajh, Xu, Muennighoff, Jurafsky, and Kiela}]{ethayarajh2024kto}
Ethayarajh, K.; Xu, W.; Muennighoff, N.; Jurafsky, D.; and Kiela, D. 2024.
\newblock Kto: Model alignment as prospect theoretic optimization.
\newblock \emph{arXiv preprint arXiv:2402.01306}.

\bibitem[{Guo et~al.(2024)Guo, Zhang, Liu, Liu, Khalman, Llinares, Rame, Mesnard, Zhao, Piot et~al.}]{guo2024direct}
Guo, S.; Zhang, B.; Liu, T.; Liu, T.; Khalman, M.; Llinares, F.; Rame, A.; Mesnard, T.; Zhao, Y.; Piot, B.; et~al. 2024.
\newblock Direct language model alignment from online ai feedback.
\newblock \emph{arXiv preprint arXiv:2402.04792}.

\bibitem[{Haarnoja et~al.(2017)Haarnoja, Tang, Abbeel, and Levine}]{haarnoja2017soft}
Haarnoja, T.; Tang, H.; Abbeel, P.; and Levine, S. 2017.
\newblock Reinforcement Learning with Deep Energy-Based Policies.
\newblock In \emph{ICML}.

\bibitem[{Haarnoja et~al.(2018)Haarnoja, Zhou, Abbeel, and Levine}]{haarnoja2018soft}
Haarnoja, T.; Zhou, A.; Abbeel, P.; and Levine, S. 2018.
\newblock Soft actor-critic: Off-policy maximum entropy deep reinforcement learning with a stochastic actor.
\newblock In \emph{International conference on machine learning}, 1861--1870. Pmlr.

\bibitem[{Hong, Lee, and Thorne(2024)}]{hong2024orpo}
Hong, J.; Lee, N.; and Thorne, J. 2024.
\newblock Orpo: Monolithic preference optimization without reference model.
\newblock \emph{arXiv preprint arXiv:2403.07691}.

\bibitem[{Kiran et~al.(2021)Kiran, Sobh, Talpaert, Mannion, Al~Sallab, Yogamani, and P{\'e}rez}]{kiran2021deep}
Kiran, B.~R.; Sobh, I.; Talpaert, V.; Mannion, P.; Al~Sallab, A.~A.; Yogamani, S.; and P{\'e}rez, P. 2021.
\newblock Deep reinforcement learning for autonomous driving: A survey.
\newblock \emph{IEEE transactions on intelligent transportation systems}, 23(6): 4909--4926.

\bibitem[{Knox and Stone(2008)}]{knox2008tamer}
Knox, W.~B.; and Stone, P. 2008.
\newblock Tamer: Training an agent manually via evaluative reinforcement.
\newblock In \emph{2008 7th IEEE international conference on development and learning}, 292--297. IEEE.

\bibitem[{Kohli, Salek, and Stoddard(2013)}]{kohli2013fast}
Kohli, P.; Salek, M.; and Stoddard, G. 2013.
\newblock A fast bandit algorithm for recommendation to users with heterogenous tastes.
\newblock In \emph{Proceedings of the AAAI Conference on Artificial Intelligence}, volume~27, 1135--1141.

\bibitem[{Li et~al.(2025)Li, Xia, Chang, and Wu}]{li2025length}
Li, G.; Xia, T.; Chang, Y.; and Wu, Y. 2025.
\newblock Length-controlled margin-based preference optimization without reference model.
\newblock \emph{arXiv preprint arXiv:2502.14643}.

\bibitem[{Meng, Xia, and Chen(2024)}]{meng2024simpo}
Meng, Y.; Xia, M.; and Chen, D. 2024.
\newblock Simpo: Simple preference optimization with a reference-free reward.
\newblock \emph{Advances in Neural Information Processing Systems}, 37: 124198--124235.

\bibitem[{Ng, Russell et~al.(2000)}]{ng2000algorithms}
Ng, A.~Y.; Russell, S.; et~al. 2000.
\newblock Algorithms for inverse reinforcement learning.
\newblock In \emph{Icml}, volume~1, 2.

\bibitem[{Ouyang et~al.(2022)Ouyang, Wu, Jiang, Almeida, Wainwright, Mishkin, Zhang, Agarwal, Slama, Ray et~al.}]{ouyang2022training}
Ouyang, L.; Wu, J.; Jiang, X.; Almeida, D.; Wainwright, C.; Mishkin, P.; Zhang, C.; Agarwal, S.; Slama, K.; Ray, A.; et~al. 2022.
\newblock Training language models to follow instructions with human feedback.
\newblock \emph{Advances in neural information processing systems}, 35: 27730--27744.

\bibitem[{Park et~al.(2024)Park, Rafailov, Ermon, and Finn}]{park2024disentangling}
Park, R.; Rafailov, R.; Ermon, S.; and Finn, C. 2024.
\newblock Disentangling length from quality in direct preference optimization.
\newblock \emph{arXiv preprint arXiv:2403.19159}.

\bibitem[{Rafailov et~al.(2024)Rafailov, Hejna, Park, and Finn}]{rafailov2024r}
Rafailov, R.; Hejna, J.; Park, R.; and Finn, C. 2024.
\newblock {From R to Q*: Your language model is secretly a Q-function}.
\newblock \emph{arXiv preprint arXiv:2404.12358}.

\bibitem[{Rafailov et~al.(2023)Rafailov, Sharma, Mitchell, Manning, Ermon, and Finn}]{rafailov2023direct}
Rafailov, R.; Sharma, A.; Mitchell, E.; Manning, C.~D.; Ermon, S.; and Finn, C. 2023.
\newblock Direct preference optimization: Your language model is secretly a reward model.
\newblock \emph{Advances in Neural Information Processing Systems}, 36: 53728--53741.

\bibitem[{Saha, Pacchiano, and Lee(2023)}]{saha2023dueling}
Saha, A.; Pacchiano, A.; and Lee, J. 2023.
\newblock Dueling rl: Reinforcement learning with trajectory preferences.
\newblock In \emph{International conference on artificial intelligence and statistics}, 6263--6289. PMLR.

\bibitem[{Schoenauer et~al.(2014)Schoenauer, Akrour, Sebag, and Souplet}]{schoenauer2014programming}
Schoenauer, M.; Akrour, R.; Sebag, M.; and Souplet, J.-C. 2014.
\newblock Programming by feedback.
\newblock In \emph{International Conference on Machine Learning}, 1503--1511. PMLR.

\bibitem[{Schulman et~al.(2017)Schulman, Wolski, Dhariwal, Radford, and Klimov}]{schulman2017proximal}
Schulman, J.; Wolski, F.; Dhariwal, P.; Radford, A.; and Klimov, O. 2017.
\newblock Proximal policy optimization algorithms.
\newblock \emph{arXiv preprint arXiv:1707.06347}.

\bibitem[{Sharifnassab et~al.(2024)Sharifnassab, Salehkaleybar, Ghiassian, Kanoria, and Schuurmans}]{sharifnassab2024soft}
Sharifnassab, A.; Salehkaleybar, S.; Ghiassian, S.; Kanoria, S.; and Schuurmans, D. 2024.
\newblock Soft Preference Optimization: Aligning Language Models to Expert Distributions.
\newblock \emph{arXiv preprint arXiv:2405.00747}.

\bibitem[{Spielman(2010)}]{spielman2010algorithms}
Spielman, D.~A. 2010.
\newblock Algorithms, graph theory, and linear equations in Laplacian matrices.
\newblock In \emph{Proceedings of the International Congress of Mathematicians 2010 (ICM 2010) (In 4 Volumes) Vol. I: Plenary Lectures and Ceremonies Vols. II--IV: Invited Lectures}, 2698--2722. World Scientific.

\bibitem[{Stiennon et~al.(2020)Stiennon, Ouyang, Wu, Ziegler, Lowe, Voss, Radford, Amodei, and Christiano}]{stiennon2020learning}
Stiennon, N.; Ouyang, L.; Wu, J.; Ziegler, D.; Lowe, R.; Voss, C.; Radford, A.; Amodei, D.; and Christiano, P.~F. 2020.
\newblock Learning to summarize with human feedback.
\newblock \emph{Advances in neural information processing systems}, 33: 3008--3021.

\bibitem[{Swamy et~al.(2024)Swamy, Dann, Kidambi, Wu, and Agarwal}]{swamy2024minimaximalist}
Swamy, G.; Dann, C.; Kidambi, R.; Wu, Z.~S.; and Agarwal, A. 2024.
\newblock A minimaximalist approach to reinforcement learning from human feedback.
\newblock \emph{arXiv preprint arXiv:2401.04056}.

\bibitem[{Todorov, Erez, and Tassa(2012)}]{todorov2012mujoco}
Todorov, E.; Erez, T.; and Tassa, Y. 2012.
\newblock Mujoco: A physics engine for model-based control.
\newblock In \emph{2012 IEEE/RSJ International Conference on Intelligent Robots and Systems}, 5026--5033. IEEE.

\bibitem[{Warnell et~al.(2018)Warnell, Waytowich, Lawhern, and Stone}]{warnell2018deep}
Warnell, G.; Waytowich, N.; Lawhern, V.; and Stone, P. 2018.
\newblock Deep tamer: Interactive agent shaping in high-dimensional state spaces.
\newblock In \emph{Proceedings of the AAAI conference on artificial intelligence}, volume~32.

\bibitem[{Wilson, Fern, and Tadepalli(2012)}]{wilson2012bayesian}
Wilson, A.; Fern, A.; and Tadepalli, P. 2012.
\newblock A bayesian approach for policy learning from trajectory preference queries.
\newblock \emph{Advances in neural information processing systems}, 25.

\bibitem[{Wu and Sun(2023)}]{wu2023making}
Wu, R.; and Sun, W. 2023.
\newblock Making rl with preference-based feedback efficient via randomization.
\newblock \emph{arXiv preprint arXiv:2310.14554}.

\bibitem[{Xie et~al.(2024)Xie, Foster, Krishnamurthy, Rosset, Awadallah, and Rakhlin}]{xie2024exploratory}
Xie, T.; Foster, D.~J.; Krishnamurthy, A.; Rosset, C.; Awadallah, A.; and Rakhlin, A. 2024.
\newblock Exploratory preference optimization: Harnessing implicit q*-approximation for sample-efficient rlhf.
\newblock \emph{arXiv preprint arXiv:2405.21046}.

\bibitem[{Xu et~al.(2024)Xu, Sharaf, Chen, Tan, Shen, Van~Durme, Murray, and Kim}]{xu2024contrastive}
Xu, H.; Sharaf, A.; Chen, Y.; Tan, W.; Shen, L.; Van~Durme, B.; Murray, K.; and Kim, Y.~J. 2024.
\newblock Contrastive preference optimization: Pushing the boundaries of llm performance in machine translation.
\newblock \emph{arXiv preprint arXiv:2401.08417}.

\bibitem[{Xu et~al.(2020)Xu, Wang, Yang, Singh, and Dubrawski}]{xu2020preference}
Xu, Y.; Wang, R.; Yang, L.; Singh, A.; and Dubrawski, A. 2020.
\newblock Preference-based reinforcement learning with finite-time guarantees.
\newblock \emph{Advances in Neural Information Processing Systems}, 33: 18784--18794.

\bibitem[{Zeng et~al.(2016)Zeng, Wang, Mokhtari, and Li}]{zeng2016online}
Zeng, C.; Wang, Q.; Mokhtari, S.; and Li, T. 2016.
\newblock Online context-aware recommendation with time varying multi-armed bandit.
\newblock In \emph{Proceedings of the 22nd ACM SIGKDD international conference on Knowledge discovery and data mining}, 2025--2034.

\bibitem[{Zhan et~al.(2023)Zhan, Uehara, Sun, and Lee}]{zhan2023provable}
Zhan, W.; Uehara, M.; Sun, W.; and Lee, J.~D. 2023.
\newblock Provable reward-agnostic preference-based reinforcement learning.
\newblock \emph{arXiv preprint arXiv:2305.18505}.

\bibitem[{Zhang, Wei, and Ying(2024)}]{zhang2024reinforcement}
Zhang, Q.; Wei, H.; and Ying, L. 2024.
\newblock Reinforcement learning from human feedback without reward inference: Model-free algorithm and instance-dependent analysis.
\newblock \emph{arXiv preprint arXiv:2406.07455}.

\bibitem[{Zhang and Ying(2024)}]{zhang2024zeroth}
Zhang, Q.; and Ying, L. 2024.
\newblock Zeroth-Order Policy Gradient for Reinforcement Learning from Human Feedback without Reward Inference.
\newblock \emph{arXiv preprint arXiv:2409.17401}.

\bibitem[{Zhu, Jordan, and Jiao(2023)}]{zhu2023principled}
Zhu, B.; Jordan, M.; and Jiao, J. 2023.
\newblock Principled reinforcement learning with human feedback from pairwise or k-wise comparisons.
\newblock In \emph{International Conference on Machine Learning}, 43037--43067. PMLR.

\bibitem[{Zhu, Jordan, and Jiao(2024)}]{zhu2024iterative}
Zhu, B.; Jordan, M.~I.; and Jiao, J. 2024.
\newblock Iterative data smoothing: Mitigating reward overfitting and overoptimization in rlhf.
\newblock \emph{arXiv preprint arXiv:2401.16335}.

\bibitem[{Ziegler et~al.(2019)Ziegler, Stiennon, Wu, Brown, Radford, Amodei, Christiano, and Irving}]{ziegler2019fine}
Ziegler, D.~M.; Stiennon, N.; Wu, J.; Brown, T.~B.; Radford, A.; Amodei, D.; Christiano, P.; and Irving, G. 2019.
\newblock Fine-tuning language models from human preferences.
\newblock \emph{arXiv preprint arXiv:1909.08593}.

\end{thebibliography}
\newpage

\newpage

\newpage
\appendix
\setcounter{secnumdepth}{2}

\section{Proof of the Theorem \ref{thm:pref-sat}}
\label{apdx:sec:theory}

\begin{assumption}[Bradley--Terry preferences on the \emph{soft}-optimal \(Q\)-function]
\label{apdx:as:BT-Qstar}
Let \(Q^{\star} : \mathcal{S}\!\times\!\mathcal{A}\!\to\!\mathbb{R}\) be the soft-optimal
state-action value function of the MDP, i.e.,

\begin{align}
Q^{\star}(s,a)\;=\;
    &\max_{\pi}\;
      \E\!\Bigl[\;
          \sum_{t=0}^{\infty}\gamma^{t}
          \bigl(r(s_t,a_t)\;+\;\lambda\,
                 \mathcal{H}\!\bigl(\pi(\,\cdot\!\mid\!s_t)\bigr)\bigr) \notag \\
    &\qquad\qquad\qquad\qquad\;\Big|\;s_{0}=s,\;a_{0}=a
      \Bigr], \notag
\end{align}
where $\mathcal{H}(.)$ is the entropy function and $\lambda$ is the entropy coefficient.
Assume that there exists a parameter \(\beta>0\) such that, for every
\(s\in\mathcal{S}\) and any \(a,b\in\mathcal{A}\),

\begin{align}
P^{\star}\!\bigl(a\succ b \mid s\bigr)
  \;=\;
  &\sigma\!\Bigl(
      \beta\,[\,Q^{\star}(s,a)-Q^{\star}(s,b)\,]
  \Bigr), \notag \\
  &\qquad
  \sigma(z)=\frac{1}{1+e^{-z}}. \notag
\end{align}
\end{assumption}

\begin{theorem}[Preference loss recovers the optimal
policy]
\label{apdx:thm:pref-sat}
Fix a state \(s\in\mathcal{S}\) and abbreviate
\(Q^{\star}_a \coloneqq Q^{\star}(s,a)\).
Suppose that \Cref{apdx:as:BT-Qstar} holds.
Under uniform sampling of \emph{ordered} pairs
\((a,b)\sim\mathrm{Unif}(\mathcal{A}^{2})\)
and the tabular full-support
parameterization
\(
  \ell_a=\log\pi(a\mid s)
\)
(\(\sum_a e^{\ell_a}=1,\;e^{\ell_a}>0\)),
consider the preference loss

\begin{align}
\loss(\boldsymbol{\ell})
  \;=\;
  & -\frac{1}{|\mathcal{A}|^{2}}
   \sum_{(a,b)\in\mathcal{A}}
   P^{\star}(a\succ b\mid s)\,
   \log\sigma\!\bigl(\alpha(\ell_a-\ell_b)\bigr), \notag \\
  &\qquad
  \alpha>0.
\end{align}

This loss is strictly convex on the set of full-support policies and is
minimized uniquely at
\begin{equation}
\pi_{\star}(a\mid s)
  \;=\;
  \frac{\exp\!\bigl(\tfrac{\beta}{\alpha}\,Q^{\star}_a\bigr)}
       {\displaystyle\sum_{a'\in\mathcal{A}}
        \exp\!\bigl(\tfrac{\beta}{\alpha}\,Q^{\star}_{a'}\bigr)}.
\label{apdx:eq:pi-minimizer}  
\end{equation}

Furthermore, this Gibbs distribution coincides with the \emph{global} maximizer of the entropy-regularized RL (or SAC objective) when \(\lambda=\alpha/\beta\):
        \begin{equation}
        \max_{\pi}\;
        \E_{\pi}\!\Bigl[
          \sum_{t=0}^{\infty}\gamma^{t}
          \bigl(r(s_t,a_t)+\lambda\,\mathcal{H}(\pi(\cdot\mid s_t))\bigr)
        \Bigr].
        \label{apdx:eq:ent-reg}            
        \end{equation}

\end{theorem}

\begin{proof}
Because the policy is tabular, we fix the state \(s\), and introduce the log-policy vector
\(
  \boldsymbol{\ell}\!=\!(\ell_a)_{a\in\setA}
\).
Define the policy and the Bradley–Terry probabilities as follows:
\[
P_{ab}(\boldsymbol{\ell})
  := \sigma\!\bigl(\alpha(\ell_a-\ell_b)\bigr),
\qquad
P^{\star}_{ab}
  := \sigma\!\bigl(\beta(Q^{\star}_a-Q^{\star}_b)\bigr),
\qquad a,b\in\setA.
\]

First, we reformulate the loss of the theorem. For this purpose, we consider two cases:

\begin{enumerate}
    \item  When \(a=b\): In this case the two logits coincide, so
\(P^{\star}_{aa}=P_{aa}=\sigma(0)=\tfrac12\); hence in this case each summand
equals \(-\tfrac12\log\tfrac12=\tfrac{\log2}{2}\).
Summing over the \(|\setA|\) therefore contributes the
constant
\(\tfrac{|{\setA}|\,\log 2}{2\,|\setA|^{2}}\) in the loss.

\item 
For any two different actions \(a\neq b\) the ordered pairs
\((a,b)\) and \((b,a)\) both appear.
Because \(\sigma(z)+\sigma(-z)=1\), we have the identities
\(P_{ba}=1-P_{ab}\) and \(P^{\star}_{ba}=1-P^{\star}_{ab}\).
Grouping those two ordered terms gives the compact expression
\begin{equation}
  \lossp(\boldsymbol{\ell})
  := -\frac{1}{|\setA|^{2}}
     \sum_{\{a,b\} \in \setA, a \neq b}
       \Bigl[
         P^{\star}_{ab}\log P_{ab}(\boldsymbol{\ell})
         +P^{\star}_{ba}\log P_{ba}(\boldsymbol{\ell})
       \Bigr]
\label{eq:loss-unordered}
\end{equation}

\end{enumerate}

Therefore, 
\(
  \loss
  = \tfrac{|{\setA}|\,\log 2}{2\,|\setA|^{2}}
  +\lossp.
\)
Since the additive constant is not used in the optimization, it can be discarded. Therefore, we may optimize
\(\lossp\) instead of \(\loss\).

Now we characterize the stationary points. For this purpose, we compute the partial derivative of $\lossp$ with respect to the $\ell_k$,
$\frac{\partial\lossp}{\partial\ell_k}$. Based on Lemma \ref{lem:pref-grad} only the terms that contain \(k\) depend on \(\ell_k\), so
\begin{equation}
\frac{\partial\lossp}{\partial\ell_k}
  = -\frac{\alpha}{|\setA|^{2}}
    \sum_{b\neq k}
      \bigl[P^{\star}_{kb}-P_{kb}(\boldsymbol{\ell})\bigr].
      \label{eq:grad-loss}
\end{equation}
A stationary point satisfies
\(\sum_{b\neq k}(P_{kb}-P^{\star}_{kb})=0\) for every \(k\).
Subtracting the same identity written for another action \(j\) yields

\[
  \underbrace{\sum_{b\neq k}\!\bigl[P_{kb}-P^{\star}_{kb}\bigr]}
             _{=\,0}
 \;-\;
  \underbrace{\sum_{b\neq j}\!\bigl[P_{jb}-P^{\star}_{jb}\bigr]}
             _{=\,0}
 \;=\;0.
\]
Expand the two sums and separate the terms that explicitly involve the
pair \((k,j)\):

\begin{align}
\bigl[P_{kj}-P^{\star}_{kj}\bigr]
+\!\!\sum_{b\notin\{k,j\}}\!\!\bigl[P_{kb}-P^{\star}_{kb}\bigr]
\;-\;
\bigl[P_{jk}-P^{\star}_{jk}\bigr] \notag \\
-\!\!\sum_{b\notin\{k,j\}}\!\!\bigl[P_{jb}-P^{\star}_{jb}\bigr]
=0.
\label{eq:diff-rows}
\end{align}

Because a Bradley–Terry probability satisfies
\(P_{jk}=1-P_{kj}\) and the same holds for \(P^{\star}\),
\(
P_{jk}-P^{\star}_{jk} = -(P_{kj}-P^{\star}_{kj}).
\)
Using this identity in \eqref{eq:diff-rows} gives
\begin{equation}
2\bigl[P_{kj}-P^{\star}_{kj}\bigr]=0,
\label{eq:2+Sigma}
\end{equation}
hence,
\[
P_{kj}=P^{\star}_{kj}.
\]
Because \(\sigma\) is strictly increasing,
\[
\ell_k-\ell_j
  = \tfrac{\beta}{\alpha}\bigl(Q^{\star}_k-Q^{\star}_j\bigr),
  \qquad\forall j,k.
\]
Therefore, there exists \(c\in\mathbb{R}\) with
\begin{equation}
\ell_a
  = c + \tfrac{\beta}{\alpha}Q^{\star}_a,
  \qquad \forall a\in\setA.
\label{eq:affine-ell}
\end{equation}

Now using \(\sum_{a}e^{\ell_a}=1\) with \eqref{eq:affine-ell} gives
\[
e^{c}
  = \Bigl(\sum_{a}\exp\!\bigl(\tfrac{\beta}{\alpha}Q^{\star}_a\bigr)\Bigr)^{-1}.
\]
Hence, the unique stationary point is as follows:
\begin{equation}
\pi_{\star}(a\mid s)
  = \frac{\exp(\tfrac{\beta}{\alpha}Q^{\star}_a)}
         {\displaystyle\sum_{a'\in\setA}
          \exp(\tfrac{\beta}{\alpha}Q^{\star}_{a'})}.
\label{eq:pi-star}
\end{equation}

If we write the KKT conditions of the loss and derive the value of the Lagrange multiplier \(\lambda\),  \(\lambda\) will be zero. Hence, the above stationary point is valid. See Lemma ~\ref{lem:kkt-zero} for the details.




\paragraph{Computing Hessian:} To compute the Hessian we define $w_{ab}$ as follows:

\[
w_{ab}:=P_{ab}(\boldsymbol{\ell})P_{ba}(\boldsymbol{\ell})
       =P_{ab}(\boldsymbol{\ell})\!\bigl[1-P_{ab}(\boldsymbol{\ell})\bigr]
       \;>\;0.
\]
Using
\[
\frac{\partial P_{ab}}{\partial\ell_a}=+\alpha\,w_{ab},
\qquad
\frac{\partial P_{ab}}{\partial\ell_b}=-\alpha\,w_{ab}.
\]
Now, if we differentiate \eqref{eq:grad-loss} once more. For (\(i\neq j\)):
\[
\frac{\partial^{2}\mathcal{L}_{\text p}}
     {\partial\ell_j\,\partial\ell_i}
   =-\frac{\alpha}{|\mathcal{A}|^{2}}
     \bigl[\alpha\,w_{ij}\bigr]
   =-\frac{\alpha^{2}}{|\mathcal{A}|^{2}}\,w_{ij}.
\]

For (\(i=j\)):
\[
\frac{\partial^{2}\mathcal{L}_{\text p}}
     {\partial\ell_i^{2}}
   =-\frac{\alpha}{|\mathcal{A}|^{2}}
     \sum_{b\neq i}
        (-\alpha\,w_{ib})
   =\frac{\alpha^{2}}{|\mathcal{A}|^{2}}
     \sum_{b\neq i} w_{ib}.
\]

Now using the above derivations, we write the matrix form of Hessian.
\[
\;
\Hess
   =\frac{\alpha^{2}}{|\mathcal{A}|^{2}}
      \bigl(\mathbf D-\mathbf W\bigr)
\qquad
\begin{aligned}
   &W_{ij}=w_{ij}\;(i\neq j),\quad W_{ii}=0,\\
   &D_{ii}=\sum_{b\neq i}w_{ib}.
\end{aligned}
\]
 To prove that $\lossp$ is strictly convex, we should prove that $\Hess$ (or $\Less$) is positive-definite. The matrix \(\Less=\Dess-\Wess\) is a weighted graph Laplacian of the
complete graph on \(\mathcal A\) as its off–diagonal entries are negative, diagonals are
positive, and each row sums to zero\citep{spielman2010algorithms}. 

For a matrix $\Less$ to be positive definite, we should have for any \(v\in\mathbb R^{|\mathcal A|}\), $v^{\!\top} \Less v>0$. In our case, one has
\begin{equation}
v^{\!\top} \Less v
  =\frac12\sum_{i,j}w_{ij}(v_i-v_j)^{2}.
    \label{eq:quadratic}
\end{equation}
Identity \eqref{eq:quadratic} follows from expanding
\(v^{\!\top}(\Dess-\Wess)v\) and re-grouping terms
(see \citet{spielman2010algorithms} for more details).
Because every weight \(w_{ij}>0\), the RHS is
non-negative, hence it is always equal to or bigger than zero. Therefore $\Less$ is positive-semidifinite ($\Less\succeq 0$) and equals to zero \emph{iff}
\[
v_1=\dots=v_{|\mathcal A|}.
\]

In other words, only subspace
\(\mathrm{span}\{\mathbf1\}
  =\{\,a\mathbf1 : a\in\R, a\neq 0\}\),
whose members have all coordinates equal
(\(v_1=\dots=v_{|\mathcal A|}\)) makes $v^{\!\top} \Less v$ equal to zero.
Now, we prove that given the constraint imposed on our problem, $v^{\!\top} \Less v$ cannot be equal to zero.

In general unconstrained optimization, \(v\) in $v^{\!\top} \Less v$, shows all possible directions in $\mathbb R^{|\mathcal A|}$.
In our case, the optimization is constrained, and the function \( \lossp \) is restricted to a constraint set \( C \) (the probability simplex: $\sum_{a\in\setA}e^{\ell_a}-1=0$), hence the condition \( v^\top \Hess \, v \geq 0 \) is only required for vectors \( v \) in the tangent space of \( C \) at \( \ell \). This is because the tangent space of a convex set \( C \) at any \( v \in C \) is the set of feasible directions within \( C \) \citep{absil2009optimization}.

The parameter set of the $\Hess$ (or accordingly $\Less$) is 
\[
\mathcal E:=\Bigl\{\boldsymbol\ell\in\mathbb R^{|\mathcal A|}\;:\;
              \sum_a e^{\ell_a}=1,\;
              e^{\ell_a}>0\Bigr\}
\]

We define
\(g(\boldsymbol\ell)=\sum_{a\in\setA}e^{\ell_a}-1=0\)
and its gradient
\(\nabla g(\boldsymbol\ell)=e^{\boldsymbol\ell}
   :=(e^{\ell_1},\dots,e^{\ell_{|\setA|}})^{\!\top}>0\).
A displacement \(v\in\R^{|\setA|}\) is \emph{feasible} iff  it is in the tangent space (denote it with \( T_{\boldsymbol\ell}\mathcal E\)) that is $\nabla g(\boldsymbol\ell)\cdot v \;=\;0$. Hence,
\[
 e^{\boldsymbol\ell}\!\cdot v \;=\;0.
\]

Now consider a vector in $\mathrm{span}\{\mathbf1\}$. For any \(v=a\mathbf 1\) with \(a\neq0\),
\[
e^{\boldsymbol\ell}\!\cdot v
  \;=\; a\,e^{\boldsymbol\ell}\!\cdot\mathbf 1
  \;=\; a\!\sum_{b\in\setA}e^{\ell_b}
  \;=\; a\;\neq\;0.
\]
Hence, \(v\notin T_{\boldsymbol\ell}\mathcal E\).
Hence, for every \emph{feasible} \(v\neq0\),
\[
v^{\!\top}\Less v \;>\;0
\quad\Longrightarrow\quad
v^{\!\top}\Hess v
  =\frac{\alpha^{2}}{|\setA|^{2}}\,v^{\!\top}\Less v \;>\;0 .
\]
Thus, the Hessian is positive–definite along all feasible directions,
which establishes the strict convexity of the preference loss on the
full support tabular policy.

\paragraph{Another way to proof the uniqueness of the solution:}

Assume, for contradiction, that there exists another log–policy vector
\(\tilde{\boldsymbol\ell}\in\mathcal E\) that also satisfies the
stationarity system
\(
  \sum_{b\neq k}(P_{kb}-P^{\star}_{kb})=0,\;\forall k
\)
and the normalization
\(\sum_{a}e^{\tilde{\ell}_a}=1\).
For every ordered pair \((k,j)\) the argument leading to
\eqref{eq:2+Sigma} then gives
\(P_{kj}(\tilde{\boldsymbol\ell})=P^{\star}_{kj}\)
as well.
Because \(\sigma\) is strictly increasing, this
implies
\[
  \tilde{\ell}_k-\tilde{\ell}_j
    \;=\;\ell_k-\ell_j,
    \qquad\forall k,j\in\mathcal A,
\]
hence
\(
  \tilde{\boldsymbol\ell}=\boldsymbol\ell+\delta\mathbf1
\)
for some \(\delta\in\R\setminus\{0\}\).
But then
\[
  \sum_{a}e^{\tilde{\ell}_a}
  \;=\;
  e^{\delta}\sum_{a}e^{\ell_a}
  \;=\;
  e^{\delta}\neq1,
\]
contradicting the constraint that every feasible
\(\boldsymbol\ell\) must satisfy \(\sum_{a}e^{\ell_a}=1\).
Therefore \(\delta=0\) and \(\tilde{\boldsymbol\ell}=\boldsymbol\ell\),
proving that the stationary point is unique.

\paragraph{Connection with soft actor-critic:}
For the same fixed state consider
\[
J_{\tau}(\pi)
  := \E_{a\sim\pi}[Q^{\star}_a] + \tau\mathcal{H}(\pi),
\qquad
\mathcal{H}(\pi):=-\sum_a\pi(a)\log\pi(a).
\]
Introducing a Lagrange multiplier \(\lambda\) for
\(\sum_a\pi(a)=1\) gives
\(Q^{\star}_a-\tau(\log\pi(a)+1)-\lambda=0\),
hence \(\pi(a)\!\propto\!\exp(Q^{\star}_a/\tau)\).
Normalization produces
\[
\pi_{\text{SAC}}(a\mid s)
  = \frac{\exp(Q^{\star}_a/\tau)}
         {\sum_{a'}\exp(Q^{\star}_{a'}/\tau)}.
\]
Choosing \(\tau=\alpha/\beta\) recovers \eqref{eq:pi-star}, so the
minimizer of $\lossp$ coincides with the soft actor-critic solution with
temperature \(\tau=\alpha/\beta\).

\end{proof}

\begin{lemma}[Gradient of the unordered-pair preference loss]
\label{lem:pref-grad}
Let
\[
  \lossp(\boldsymbol{\ell})
    := -\frac{1}{|\setA|^{2}}
       \sum_{\{a,b\} \in \setA, a \neq b} 
         \Bigl[
           P^{\star}_{ab}\log P_{ab}(\boldsymbol{\ell})
          +P^{\star}_{ba}\log P_{ba}(\boldsymbol{\ell})
         \Bigr],
\]
where \(P^{\star}_{ab}\), $P_{ab}(\boldsymbol{\ell})$ and $\boldsymbol{\ell}$ are defined in Theorem \ref{apdx:thm:pref-sat}.
Then, for every action \(k\in\setA\),
\begin{equation}
\frac{\partial\lossp}{\partial\ell_k}
  =-\frac{\alpha}{|\setA|^{2}}
    \sum_{b\neq k}\bigl[P^{\star}_{kb}-P_{kb}(\boldsymbol{\ell})\bigr].
\end{equation}
\end{lemma}

\begin{proof}
For each unordered pair \(\{a,b\}\), define
\[
  g_{ab}(\boldsymbol{\ell})
    :=P^{\star}_{ab}\log P_{ab}(\boldsymbol{\ell})
      +P^{\star}_{ba}\log P_{ba}(\boldsymbol{\ell}).
\]
Because
\(
  \dfrac{d}{dz}\log\sigma(z)=1-\sigma(z)
\)
and
\(
  \partial(\ell_a-\ell_b)/\partial\ell_k
    =\mathbf1\{k=a\}-\mathbf1\{k=b\},
\)
\[
\frac{\partial}{\partial\ell_k}\log P_{ab}
  =\alpha(1-P_{ab})\bigl[\mathbf1\{k=a\}-\mathbf1\{k=b\}\bigr],
\]\[
\frac{\partial}{\partial\ell_k}\log P_{ba}
  =\alpha P_{ab}\bigl[\mathbf1\{k=b\}-\mathbf1\{k=a\}\bigr].
\]
Since \(P^{\star}_{ba}=1-P^{\star}_{ab}\) and \(P_{ba}=1-P_{ab}\),
\begin{equation}
\frac{\partial g_{ab}}{\partial\ell_k}
  =\alpha(\mathbf1\{k=a\}-\mathbf1\{k=b\})
     \bigl[P^{\star}_{ab}-P_{ab}\bigr].
     \label{eq:gab-grad}
\end{equation}
Insert \eqref{eq:gab-grad} into the loss and sum over all unordered pairs that
contain \(k\):
\[
\frac{\partial\lossp}{\partial\ell_k}
  =-\frac{\alpha}{|\setA|^{2}}
    \sum_{\{a,b\}}(\mathbf1\{k=a\}-\mathbf1\{k=b\})
              \bigl[P^{\star}_{ab}-P_{ab}\bigr].
\]
If \(k=a\) (and \(b>k\)) the indicator equals \(+1\); if \(k=b\) (with
\(a<k\)) it equals \(-1\).
Using again the symmetry
\(P^{\star}_{ak}=1-P^{\star}_{ka}\) and \(P_{ak}=1-P_{ka}\),
the negative sign flips the difference so that both cases contribute the
same quantity \(P^{\star}_{kb}-P_{kb}\).  Hence
\[
\frac{\partial\lossp}{\partial\ell_k}
  =-\frac{\alpha}{|\setA|^{2}}
    \sum_{b\neq k}\bigl[P^{\star}_{kb}-P_{kb}(\boldsymbol{\ell})\bigr],
\]
completing the proof.
\end{proof}

\begin{lemma}[The KKT multiplier]\label{lem:kkt-zero}
Consider the constrained minimization of the unordered-pair preference
loss
\[
  \min_{\boldsymbol{\ell}\in\mathbb{R}^{|\mathcal A|}}
        \lossp(\boldsymbol{\ell})
  \quad\text{s.t.}\quad
  g(\boldsymbol{\ell}) := \sum_{a\in\mathcal A}e^{\ell_a}-1 = 0,
\]
with $\lossp$ defined in \eqref{eq:loss-unordered}.  
Let $\lambda\in\mathbb R$ be the Lagrange multiplier associated with the
normalization constraint.  
At every KKT point $(\boldsymbol{\ell},\lambda)$ one necessarily has
\[
  \lambda \;=\; 0 .
\]
\end{lemma}

\begin{proof}
The KKT stationarity condition for each action $k\in\mathcal A$ is
\[
  -\frac{\alpha}{|\mathcal A|^{2}}
      \sum_{b\neq k}\bigl[P^{\star}_{kb}-P_{kb}(\boldsymbol{\ell})\bigr]
  \;+\;\lambda\,e^{\ell_k}=0 .
\tag{\theequation}\label{eq:kkt-stat}
\]
Summing \eqref{eq:kkt-stat} over all $k$ gives
\begin{equation}
  -\frac{\alpha}{|\mathcal A|^{2}}
      \sum_{k}\sum_{b\neq k}
        \bigl[P^{\star}_{kb}-P_{kb}(\boldsymbol{\ell})\bigr]
  + \lambda\sum_{k}e^{\ell_k}=0 .
  \label{eq:kkt-derivative-sum}
\end{equation}

Because the log–policy variables satisfy the equality constraint
\(\sum_{k}e^{\ell_k}=1\),  
the second term in \eqref{eq:kkt-derivative-sum} sums to \(\lambda\).

Rewrite the double sum by grouping every \emph{ordered} pair
\((k,b)\) with its reverse \((b,k)\):
\[
  \sum_{k}\,\sum_{b\neq k}
      \bigl[P^{\star}_{kb}-P_{kb}\bigr]
  \;=\;
  \sum_{\substack{k,b\in\mathcal{A}\\ k<b}}
      \Bigl[(P^{\star}_{kb}-P_{kb})
            +(P^{\star}_{bk}-P_{bk})\Bigr].
\]
Using the fact that
\(P^{\star}_{kb}+P^{\star}_{bk}=1\) and
\(P_{kb}+P_{bk}=1\),
each term in the bracket equals to \(1-1=0\).  
Therefore, the entire double sum is zero, and we can imply that $\lambda=0$.

\end{proof}


\section{Trajectory–Level Analysis of DFA}
\label{sec:traj-analysis}

Let $\mathcal{T}_H$ be the (finite) set of all length–$H$ trajectories
$
  \tau=(s_{0},a_{0},\ldots ,s_{H-1},a_{H-1})
$
that can be generated by the MDP.

\begin{assumption}[Trajectory-level Bradley–Terry model]
\label{ass:TBT}
Let
\[
  G^{\star}(\tau)
  \;:=\;
  \sum_{t=0}^{H-1}\gamma^{t}\,
     \Bigl(r(s_t,a_t)
            +\lambda\,\mathcal H\!\bigl(\pi^{\star}(\,\cdot\!\mid\!s_t)\bigr)
     \Bigr)
\]
be the {\em soft–optimal return} of trajectory $\tau$ under the entropy
coefficient $\lambda>0$.  
There exists $\beta>0$ such that for every pair
$\tau_1,\tau_2\in\mathcal{T}_H$
\[
  P^{\star}\!\bigl(\tau_1\succ\tau_2\bigr)
  \;=\;
  \sigma\!\bigl(\beta\,[G^{\star}(\tau_1)-G^{\star}(\tau_2)]\bigr),
  \qquad
  \sigma(z)=\tfrac{1}{1+e^{-z}}.
\]
\end{assumption}

\subsection{Trajectory preference loss}

Parameterise a \emph{trajectory-tabular} policy by one log-likelihood per
path,
$
  L_{\tau} = \log\pi_\theta(\tau),
$
subject to the simplex constraint
$
  \sum_{\tau\in\mathcal{T}_H} e^{L_\tau}=1,\;
  e^{L_\tau}>0.
$
For ordered trajectory pairs sampled uniformly from
$\mathcal{T}_H^{\,2}$ define the loss
\begin{align}
  \mathcal{L}_{\mathrm{traj}}(L)
  := -\frac{1}{|\mathcal{T}_H|^{2}}
     \sum_{\tau_1,\tau_2\in\mathcal{T}_H}
       P^{\star}(\tau_1\!\succ\!\tau_2)\,
       \log\sigma\!\bigl(\alpha[L_{\tau_1}-L_{\tau_2}]\bigr),
  \label{eq:traj-loss}
\end{align}
with parameter $\alpha>0$.

\begin{theorem}[Optimal policy for trajectory loss]
\label{thm:traj-opt}
Assume \Cref{ass:TBT} and uniform sampling of ordered trajectory pairs.
The loss \eqref{eq:traj-loss} is strictly convex on the probability
simplex $\{\!L:\sum_{\tau}e^{L_\tau}=1\!\}$ and attains its \emph{unique}
minimum at the Gibbs distribution
\begin{equation}
  \pi_{\star}(\tau)
  \;=\;
  \frac{\exp\!\bigl(\tfrac{\beta}{\alpha}\,G^{\star}(\tau)\bigr)}
       {\displaystyle\sum_{\tau'\in\mathcal{T}_H}
          \exp\!\bigl(\tfrac{\beta}{\alpha}\,G^{\star}(\tau')\bigr)}.
  \label{eq:gibbs-traj}
\end{equation}

The proof is similar to the proof of Theorem \ref{thm:pref-sat}
\end{theorem}

\subsection{Connection between the State-wise and Trajectory-wise Optima}
\label{sec:state-vs-traj}

The soft value function satisfies
\begin{equation}
  V^{\star}(s)
  \;=\;
  \lambda\,
  \log\!\sum\nolimits_{a\in\mathcal A}
        \exp\!\Bigl(\tfrac{1}{\lambda}\,Q^{\star}(s,a)\Bigr),
\label{eq:softV}
\end{equation}
while soft Bellman consistency gives, for every
state–action pair \((s,a)\),
\begin{equation}
  Q^{\star}(s,a)
  = r(s,a) 
    + \gamma\,\mathbb E_{s'\sim P(\cdot\mid s,a)}V^{\star}(s')
    + \lambda\,\mathcal H\!\bigl(\pi^{\star}(\,\cdot\!\mid s)\bigr).
\label{eq:softBellman}
\end{equation}
See \cite{haarnoja2017soft} for the proofs.
considering $\lambda = \alpha/\beta $ for every state $s$,
\begin{equation}
  \sum_{a\in\mathcal A}
     \exp\!\bigl(\tfrac{\beta}{\alpha}Q^{\star}(s,a)\bigr)
  \;=\;
  \exp\!\bigl(\tfrac{\beta}{\alpha}V^{\star}(s)\bigr).
  \label{eq:lamdaQV}
\end{equation}

Now, consider a trajectory
\(
  \tau=(s_0,a_0,\dots,s_{H-1},a_{H-1})
\).
Multiplying the optimal policy in Theorem \ref{thm:pref-sat}, along~\(\tau\) and using Equation \eqref{eq:lamdaQV} gives
\begin{align}
  \pi_{\mathrm{traj}}(\tau)& =\prod_{t=0}^{H-1}\pi_{\mathrm{st}}(a_t\mid s_t)
  \;\\ & = \;
  \exp\!\Bigl(
        \tfrac{\beta}{\alpha}
        \sum_{t=0}^{H-1}
          [\,Q^{\star}(s_t,a_t)-V^{\star}(s_t)]
      \Bigr),
\label{eq:chain}
\end{align}
where $\pi_{\mathrm{st}}$ is the optimal policy in \eqref{eq:pi-minimizer}.
From \eqref{eq:softBellman},
\(
  Q^{\star}(s_t,a_t)-V^{\star}(s_t)
  = r_t+ \lambda\mathcal H(\pi^{\star}(\cdot\mid s_t))
    -\gamma V^{\star}(s_{t+1})
\).
Summing this over \(t=0{:}H-1\) cancels the
\(V^{\star}\)-terms (telescoping), leaving
\begin{equation}
  \sum_{t=0}^{H-1}
    [\,Q^{\star}(s_t,a_t)-V^{\star}(s_t)]
  \;=\;
  G^{\star}(\tau).
\label{eq:telescope}
\end{equation}
Insert \eqref{eq:telescope} into \eqref{eq:chain}:
\begin{equation}
  \pi_{\mathrm{traj}}(\tau)=\prod_{t=0}^{H-1}\pi_{\mathrm{st}}(a_t\mid s_t)
  \;=\;
  \exp\!\bigl(\tfrac{\beta}{\alpha}G^{\star}(\tau)\bigr),
\label{eq:propto}
\end{equation}
Which is the nominator in \eqref{eq:gibbs-traj}. Equation~\eqref{eq:propto} shows that the joint likelihood assigned to \(\tau\) by the product of the per-state optimizers is proportional to
the optimal policy in \eqref{eq:gibbs-traj}.

\section{Experiments }
\label{apdx:hyperparameters}
Walker is a planar biped with four actuated joints that must walk without tipping over; Hopper is a one-leg, three-joint robot that learns to hop forward; Humanoid is a 17-joint 3D figure that must walk quickly while remaining upright; Swimmer is a three-link snake that propels itself through a viscous medium; Inverted Pendulum tasks a cart with balancing an upright pole; and MountainCar Continuous challenges a car trapped between two hills to climb the right hill by building momentum.

For GridWorld game, the grid is 5*5, where the agent starts at position 2*2. All methods use a tabular softmax policy parameterization, where each state-action pair has a corresponding logit parameter. For the reward model in RM+PPO, we use a simple tabular representation that assigns a value to each state-action pair. All methods are trained for the same number of environment interactions to ensure fair comparison. We use Adam optimizer with learning rate $3 \times 10^{-2}$ across all methods. 

Below we illustrate the results for more complex environments and different numbers of $M$ mentioned in Section \ref{sec:experiments}. We run the algorithms for 5 different seeds $\{3, 1, 14, 4, 50\}$.

\begin{figure*}[t]
    \centering
        \centering
        \includegraphics[width=.8\textwidth]{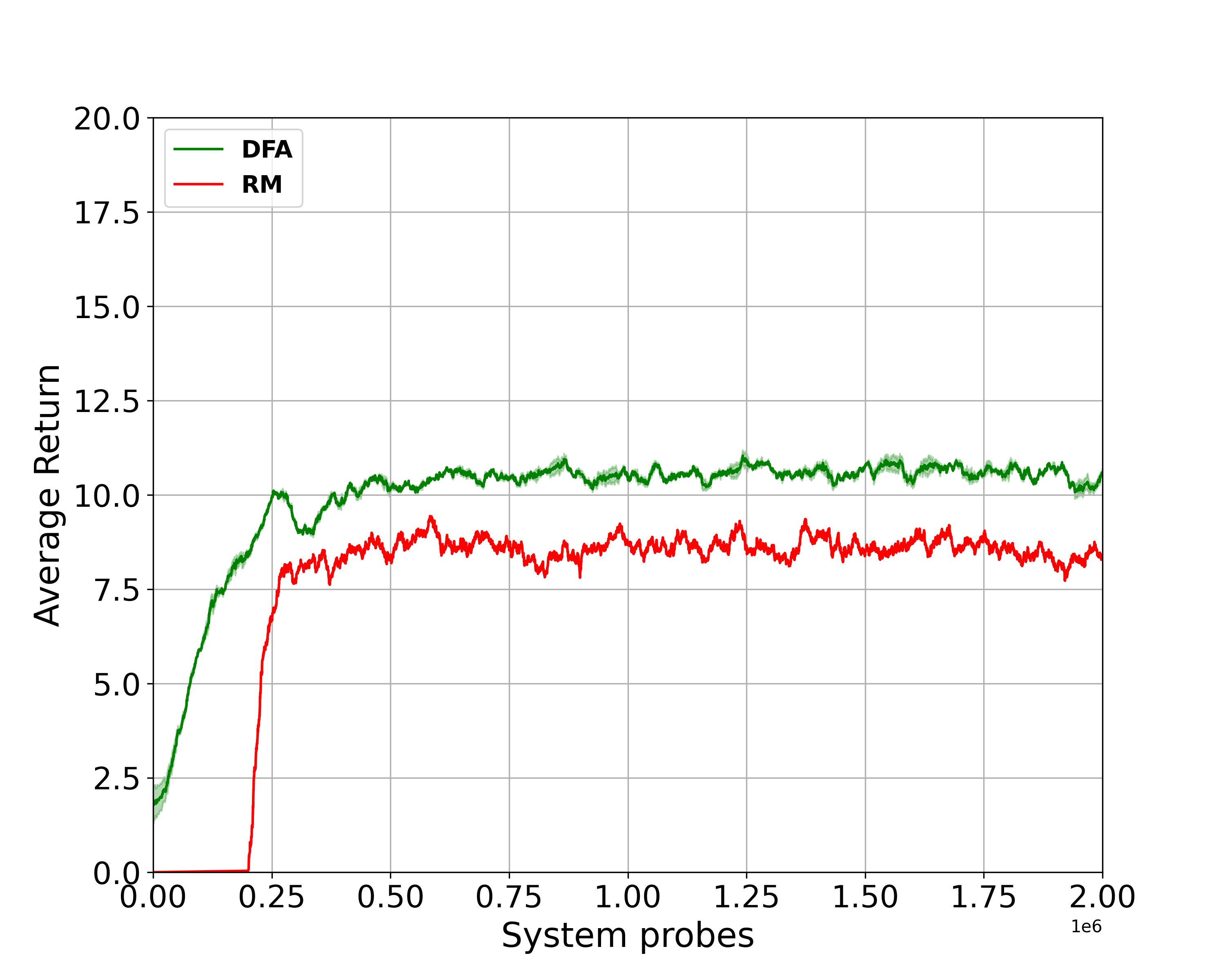}
    \caption{GridWorld results with size 10*10. $(M=500)$}
    \label{fig:grid10}
\end{figure*}

\begin{figure*}[t]
    \centering
        \centering
        \includegraphics[width=0.8\textwidth]{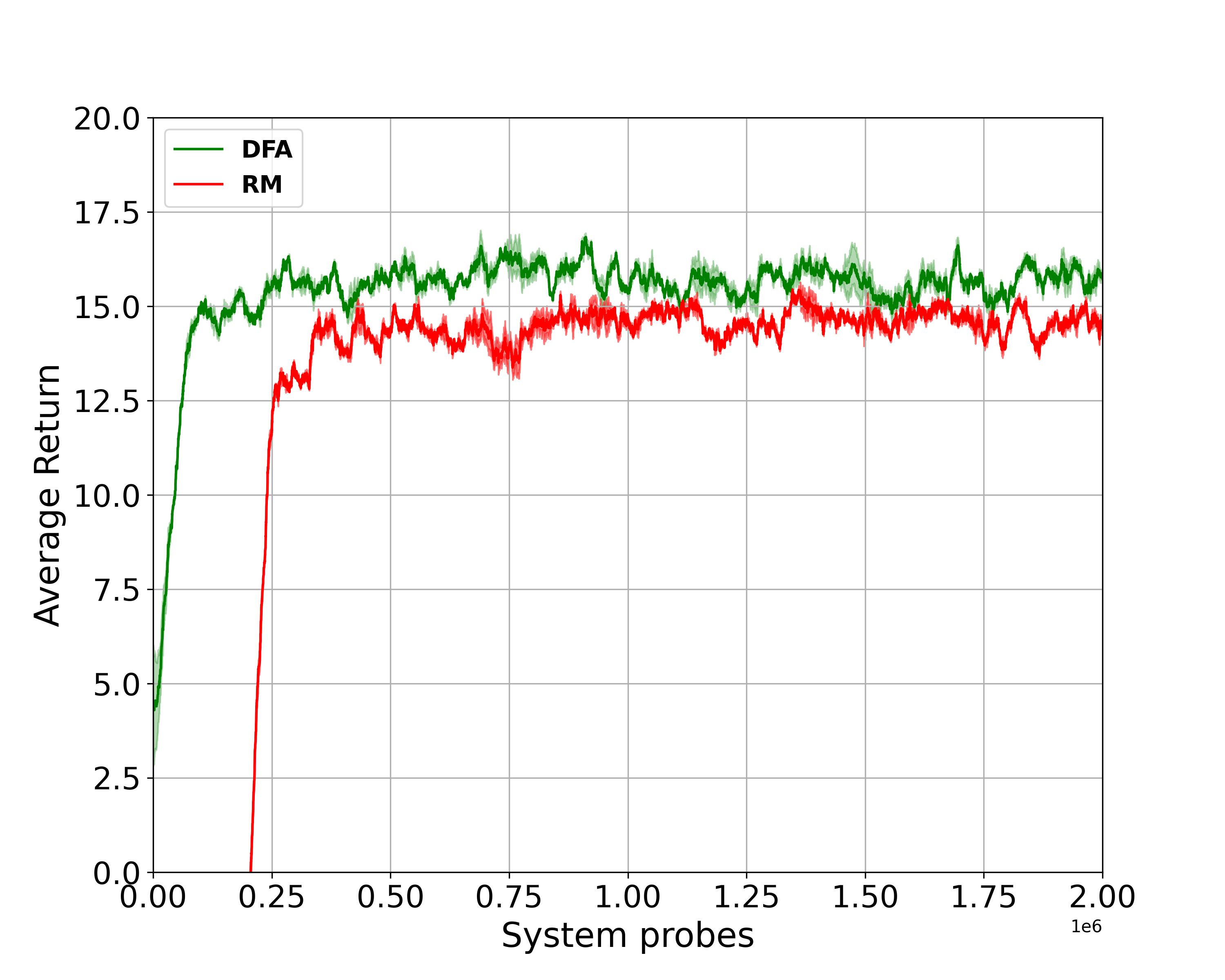}
    \caption{GridWorld results with size 20*20$(M=100)$.}
    \label{fig:grid20}
\end{figure*}

\begin{figure*}[t]
    \centering
        \centering
        \includegraphics[width=0.8\textwidth]{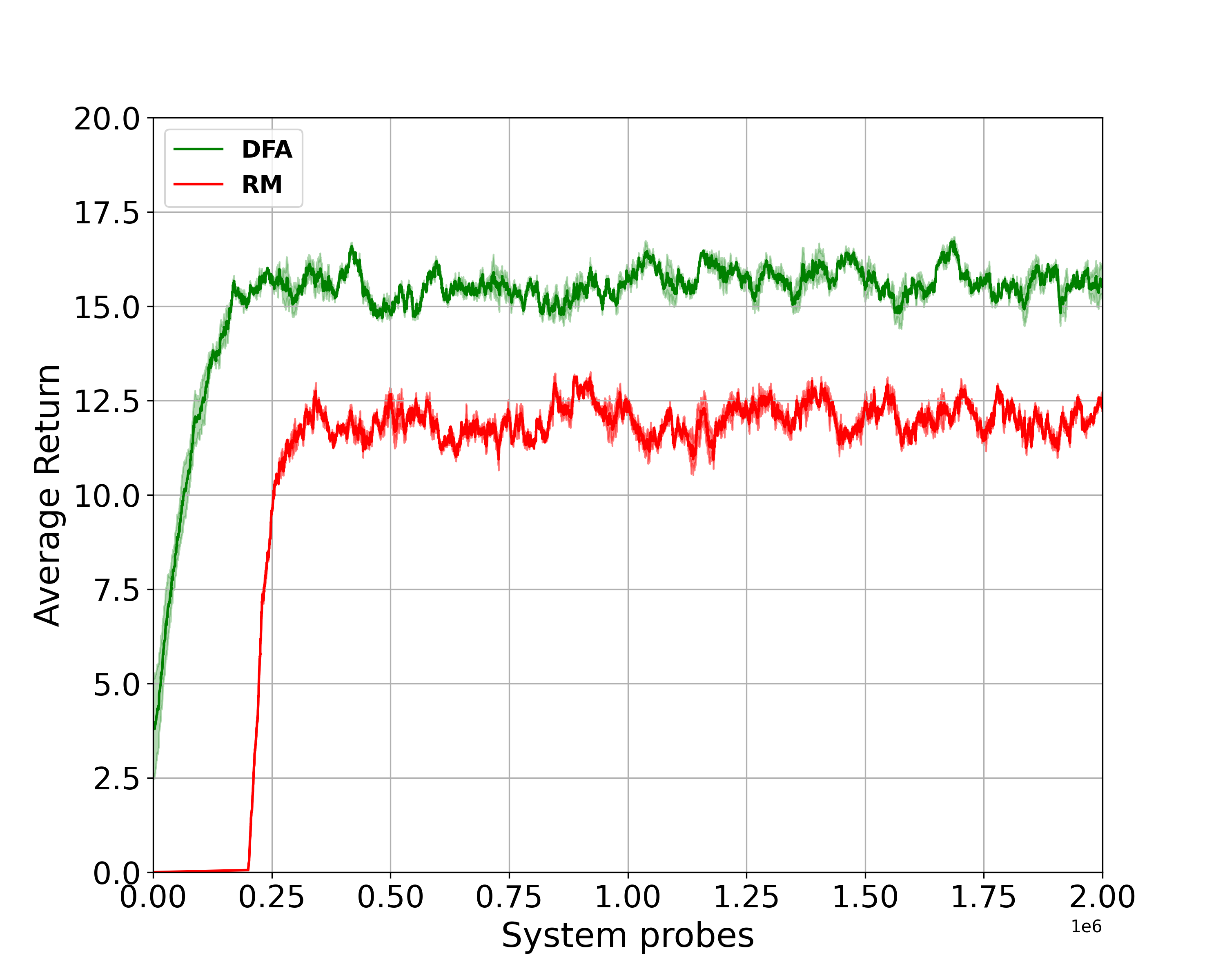}
    \caption{GridWorld results with size $20 \times 20 (M=1)$.}
    \label{fig:grid20-m1}
\end{figure*}

We utilized a Linux server with Intel Xeon CPU E5-2680 v3 (24 cores) operating at 2.50GHz with 377 GB DDR4 of memory and Nvidia Titan X Pascal GPU. The computation was distributed over 48 threads to ensure a relatively efficient run time. In our control-task experiments, DFA required more wall-clock time than SAC. For example, in the Pendulum, running 10 million system probes took 8 hours (on average) with DFA compared to 6.5 hours with SAC. In the Swimmer environment, SAC completed in 8 hours, while DFA took 9 hours. Although DFA generally requires more wall-clock time per step, in some environments (e.g., MountainCar, Swimmer) it converges in fewer steps. As a result, the increased per-step runtime does not significantly impact its overall efficiency.

For hyperparameter tuning, we performed a grid search, systematically exploring a predefined range of values for each parameter. In the following tables, we provide the fine-tuned parameters for each algorithm and method. Batch sizes are considered the same for all algorithms. The discount factor is also set to $0.99$ for all the runs.

\begin{table*}[htbp]
\centering

\begin{tabular}{lll}
\textbf{Algorithm} & \textbf{Hyper-parameter} & \textbf{Value} \\
\midrule
\multirow{7}{*}{ZPG}%
 & $T$ (iterations)                  & 1000000 \\
 & $N$ (pairs / iter)                & 1 - 10 \\
 & $M$ (votes / query)               & 1000 \\
 & $\mu$ (perturbation radius)       & 0.1 \\
 & $\alpha$ (learning-rate)          & 0.05 \\
 & $\text{trim}$ (prob.\ clip)       & $10^{-2}$ \\
\midrule
\multirow{7}{*}{RM-PPO}%
 & $\text{traj\_pairs (pretaining)}$              & 5000 \\
 & $\text{ppo\_iters}$               & 1000 \\
 & $\beta_{\mathrm{KL}}$             & 0.1 \\
 & $\gamma$ (discount)               & 1.0 \\
 & $\lambda$ (GAE)                   & 0.95 \\
\midrule
\multirow{6}{*}{DFA (on-policy)}%
\\ \\
 & $\alpha$ (temperature)             & $1\times10^{-3} - 1\times10^{-6}$ \\
 & $N_{\text{pairs/iter}}$                & 1 \\
 & $\text{iters}$                    & 100000 \\
\midrule
\multirow{6}{*}{Oracle-PPO}%
 & $\text{ppo\_iters}$               & 100000 \\
 & $\beta_{\mathrm{KL}}$             & 0.1 \\
 & $\gamma$ (discount)               & 1.0 \\
 & $\lambda$ (GAE)                   & 0.95 \\
\bottomrule
\end{tabular}

\caption{Default hyper-parameters for all algorithms used in the $5\times5$ GridWorld experiments}
\end{table*}

\begin{table*}[htbp]
\centering

\small
\setlength{\tabcolsep}{6pt}
\hspace{-2cm}
\begin{tabular}{llcccccc}
 \textbf{Alg.} & \textbf{Hyper-parameter} &
\textbf{Walker2d} & \textbf{Hopper} & \textbf{Swimmer} &
\textbf{Humanoid} & \textbf{Mountain\-CarC} & \textbf{Pendulum} \\ \midrule
\multirow{10}{*}{SAC}
 & Hidden layer size           & 64 & 64 & 64 & 64 & 64 & 64 \\
 & Policy learning-rate     & $1\times10^{-3}$ & $1\times10^{-3}$ &
   $1\times10^{-3}$ & $1\times10^{-3}$ & $1\times10^{-3}$ & $1\times10^{-3}$ \\
 & $Q$ learning-rate         & $1\times10^{-3}$ & $1\times10^{-3}$ &
   $1\times10^{-3}$ & $1\times10^{-3}$ & $1\times10^{-3}$ & $1\times10^{-3}$ \\
 & Batch size                         & 256 & 256 & 256 & 256 & 256 & 256 \\
 & Replay-buffer capacity             & 20\,000 & 20\,000 & 20\,000 & 20\,000 &
   20\,000 & 20\,000 \\
 & Entropy temperature $\lambda$       & 0.1 & 0.2 & 0.01 & 0.01 & 0.1 & 0.2 \\
 & Discount factor $\gamma$           & 0.99 & 0.99 & 0.99 & 0.99 & 0.99 & 0.99 \\
 & Soft-update coefficient $\tau$     & 0.1 & 0.005 & 0.1 & 0.1 & 0.01 & 0.005 \\
 & \# parallel envs $N_{\text{env}}$  & 32 & 32 & 32 & 32 & 32 & 32 \\
 & Training episodes                  & 50\,000 & 50\,000 & 50\,000 & 50\,000 &
   50\,000 & 50\,000 \\ \midrule
\multirow{11}{*}{DFA}
 & Hidden layer size           & 64 & 64 & 64 & 64 & 64 & 64 \\
 & Policy learning-rate     & $1\times10^{-3}$ & $1\times10^{-3}$ &
   $1\times10^{-3}$ & $1\times10^{-3}$ & $1\times10^{-3}$ & $1\times10^{-3}$ \\
 & $Q$ learning-rate         & $1\times10^{-3}$ & $1\times10^{-3}$ &
   $1\times10^{-3}$ & $1\times10^{-3}$ & $1\times10^{-3}$ & $1\times10^{-3}$ \\
 & Batch size                         & 256 & 256 & 256 & 256 & 256 & 256 \\
 & Replay-buffer capacity             & 20\,000 & 20\,000 & 20\,000 & 20\,000 &
   20\,000 & 20\,000 \\
 & Entropy temperature $\lambda$       & 0.01 & 0.1 & 0.01 & 0.01 & 0.01 & 0.01 \\
 & Temperature $\alpha$            & 0.2 & 0.2 & 0.3 & 0.2 & 0.4 & 0.2 \\
 & Discount factor $\gamma$           & 0.99 & 0.99 & 0.99 & 0.99 & 0.99 & 0.99 \\
 & Soft-update coefficient $\tau$     & 0.1 & 0.005 & 0.1 & 0.1 & 0.01 & 0.005 \\
 & \# parallel envs $N_{\text{env}}$  & 32 & 32 & 32 & 32 & 32 & 32 \\
 & Training episodes                  & 50\,000 & 50\,000 & 50\,000 & 50\,000 &
   50\,000 & 50\,000 \\
\bottomrule
\end{tabular}

\caption{Hyper-parameters for SAC and DFA across all evaluated environments}\end{table*}

\label{apdx:broader-impact}

DFA aims to make reinforcement learning from human feedback more sample-efficient by blending numeric rewards with pairwise preferences.  
Positive impacts include lowering annotation costs, enabling faster prototyping of assistive robots, and providing a simple baseline for preference-centric alignment research.  
However, the method also amplifies whatever biases or inconsistencies are present in the collected preferences: if early $Q$ estimates or human labels encode unfair or unsafe behavior, DFA may reinforce those patterns more quickly than reward-only training.  
Because DFA can learn from very small amounts of feedback, malicious or accidental injection of adversarial comparisons could steer policies toward harmful objectives—especially in safety-critical domains such as autonomous driving or content recommendation.  
The work uses only simulated environments and involves no personal data; nevertheless, broader deployment should respect fairness guidelines and, when real users provide feedback, comply with relevant privacy regulations.

\end{document}